\documentclass{article}
\PassOptionsToPackage{numbers, compress}{natbib}
\usepackage[preprint]{neurips_2026}
\usepackage[utf8]{inputenc}
\usepackage[T1]{fontenc}
\usepackage{hyperref}
\usepackage{url}
\usepackage{booktabs}
\usepackage{amsfonts}
\usepackage{nicefrac}
\usepackage{microtype}
\usepackage{xcolor}

\usepackage{amsthm}
\newtheorem{theorem}{Theorem}
\newtheorem{lemma}{Lemma}
\newtheorem{proposition}{Proposition}

\newtheorem{corollary}{Corollary}
\newtheorem{definition}{Definition}

\newtheorem{assumption}{Assumption}[section]

\usepackage{cite}
\usepackage{enumerate}
\usepackage{graphics}
\usepackage{epsfig}
\usepackage{makecell}
\usepackage{wrapfig}
\usepackage{algpseudocode}
\usepackage[linesnumbered, ruled]{algorithm2e}
\SetKwRepeat{Do}{do}{while}
\usepackage{listings}
\usepackage{graphicx}
\usepackage{subfigure}
\usepackage{bbm}
\usepackage{enumitem}
\usepackage[export]{adjustbox}
\usepackage{changes}
\usepackage{multirow}
\usepackage{extarrows}
\usepackage[capitalize,noabbrev]{cleveref}
\usepackage{threeparttable}
\usepackage{marvosym}

\usepackage{color}
\usepackage{colortbl}
\definecolor{myorange}{HTML}{E7730D}
\definecolor{myblue}{HTML}{4594c1}
\usepackage{enumitem}
\usepackage[most]{tcolorbox}

\newcommand{\example}[1]{
\begin{tcolorbox}[
    enhanced,
    colback=white,
    colframe=white,
    leftrule=0.4mm,
    rightrule=0.4mm,
    toprule=0.4mm,
    bottomrule=0.4mm,
    arc=0mm,
    left=0pt,
    right=0pt,
    top=1pt,
    bottom=1pt,
    breakable,
    borderline north={0.4mm}{0pt}{myblue},
    borderline south={0.4mm}{0pt}{myblue}
]
{\textbf{\textcolor{myblue}{Example:}}} #1
\end{tcolorbox}
}

\newcommand{\contribution}[1]{
\begin{tcolorbox}[
    enhanced,
    colback=myblue!8!white,
    colframe=myblue,
    leftrule=2mm,
    rightrule=0mm,
    toprule=0mm,
    bottomrule=0mm,
    arc=0mm,
    left=5pt,
    right=5pt,
    top=2.5pt,
    bottom=2.5pt,
    breakable,
    leftlower=2mm,
    leftupper=2mm
]
\normalsize 
#1
\end{tcolorbox}
}

\title{Power Reinforcement Post-Training of Text-to-Image Models with Super-Linear Advantage Shaping}

\author{
Haoyuan Sun$^{1}\textsuperscript{*}$,
Jing Wang$^{2}\textsuperscript{*}$, 
Yuxin Song$^{2}\textsuperscript{*$\dagger$}$, 
Yu Lu$^{3}$, Bo Fang$^{4}$, Yifu Luo$^{5}$, \\
\textbf{Jun Yin$^{5}$, Pengyu Zeng$^{5}$, Miao Zhang$^{6}$, Tiantian Zhang$^{5}$, Xueqian Wang$^{5}$, Shijian Lu$^{1}\textsuperscript{\Letter}$} \\
$^{*}$ Equal Contribution \quad $^{\dagger}$ Project Leader \quad $^\text{\Letter}$ Corresponding Author \\
$^1$ Nanyang Technological University \quad
$^2$ Baidu Inc. \quad
$^3$ Zhejiang University \quad  \\
$^4$ City University of Hong Kong \quad
$^5$ Tsinghua University \quad
$^6$ Jimei University \quad  
}

\begin{document}

\maketitle

\begin{abstract}
Recently, post-training methods based on reinforcement learning, with a particular focus on Group Relative Policy Optimization (GRPO), have emerged as the robust paradigm for further advancement of text-to-image (T2I) models. However, these methods are often prone to reward hacking, wherein models exploit biases in imperfect reward functions rather than yielding genuine performance gains. In this work, we identify that normalization could lead to miscalibration and directly removing the prompt-level standard deviation term yields an optimal policy ascent direction that is linear in the advantage but still limits the separation of genuine signals from noise. To mitigate the above issues, we propose Super-Linear Advantage Shaping (SLAS) by revisiting the functional update from an information geometry perspective. By extending the Fisher–Rao information metric with advantage-dependent weighting, SLAS introduces a non-linear geometric structure that reshapes the local policy space. This design relaxes constraints along high-advantage directions to amplify informative updates, while tightening those in low-advantage regions to suppress illusory gradients. In addition, batch-level normalization is applied to stabilize training under varying reward scales. Extensive evaluations demonstrate that SLAS consistently surpasses the DanceGRPO baseline across multiple backbones and benchmarks. In particular, it yields faster training dynamics, improved out-of-domain performance on GenEval and UniGenBench++, and enhanced robustness to model scaling, while mitigating reward hacking and preserving semantic and compositional fidelity in generations.
\end{abstract}

\section{Introduction}
\label{Introduction}
Recent advances in flow-matching diffusion modeling \citep{lipman2022flow} have established a new frontier for visual generation. Image synthesis has reached a new level of maturity with the emergence of leading models such as  Stable Diffusion 3.5 \citep{esser2024scaling}, FLUX.1 \citep{flux2024}, and Qwen-Image \citep{wu2025qwen}, which collectively define the current state-of-the-art. Simultaneously, the success of reinforcement learning (RL) in enhancing large language model (LLM) reasoning \citep{guo2025deepseek} has catalyzed efforts to transfer GRPO-style training paradigms to diffusion models, as exemplified by Flow-GRPO \citep{liu2025flow} and DanceGRPO \citep{xue2025dancegrpo}.

Although these methods achieve noticeable improvements on certain metrics or benchmarks, they still suffer from the issue of reward hacking. Instead of promoting true improvements in image quality or semantic alignment, reward hacking encourage the model to capitalize on biases in the reward model \citep{duan2025badreward}. Consequently, the model could adopt shortcut solutions that inflate the reward signal, including overemphasis on reward-favored visual patterns or reliance on spurious correlations, leading to perceptually implausible or misaligned generations. As demonstrated in \citep{wang2025pref}, reward hacking can be attributed to minimal reward differences among generated images (illustrative examples in Appendix~\ref{Illustrative Examples of Low-Variance Property}), leading to illusory advantages. Based on this, let's consider the following example:
\example{We assume that CLIPScore is employed as the reward model. First, consider a group of reward values with noticeable inter-sample gaps [0.360,0.370,0.380,0.390], which exhibit clear separation in CLIPScore assessments, yielding group-relative advantages of [-1.342,-0.447,0.447,1.342]. By contrast, for another group with far narrower inter-sample gaps [0.366, 0.367, 0.368, 0.369], the group-relative advantages remain exactly identical at [-1.342, -0.447, 0.447, 1.342], despite the substantially reduced variance in raw rewards. The former group is more likely to lead to meaningful performance gains, whereas the latter may be attributed to stochastic noise.}  
Reflecting on this example, we argue that beyond inherently minimal reward differences, reward hacking under this paradigm arises from normalization-induced miscalibration. This miscalibration assigns comparable advantage values to two distinct types of samples: those that yield genuine training benefits, and those whose improvements stem from noise in the imperfect reward model. Formally, we present a detailed formulation in \Cref{standard_GRPO}, motivating improvements along three axes: data, reward design, and algorithm design. Furthermore, we focus on the algorithm design, specifically adjusting the prompt-level standard deviation term to reduce advantage overestimation and mitigate reward hacking. To this end, we first consider removing the prompt-level standard deviation term. By applying a Taylor expansion, we derive a local functional optimization problem that characterizes the proposed update. Solving this problem yields the optimal policy update in \Cref{optimal ascent}, from which we show that maximizing the expected return functional induces an ascent direction that is linear in the advantage function class.

Therefore, although removing the prompt-level standard deviation term mitigates miscalibration, its inherent linearity still limits the ability to separate informative advantages from noise under imperfect reward models. This motivates revisiting the functional local update problem in \Cref{functional local update problem} from an information geometry perspective \citep{ay2017information}. Based on this, we construct the $\gamma$-weighted variational metric on the tangent space of the probability simplex, extending the Fisher–Rao information metric to incorporate advantage-dependent geometry. This introduces a non-linear geometric structure designed to overcome the linearity bottleneck of prior approach in distinguishing genuine benefits from noise. Reshaping the local geometry of the policy space, it amplifies high-advantage update directions while suppressing noisy low-advantage ones. Through rigorous derivation in \Cref{optimal_ascent_super}, we can obtain the \textbf{S}uper-\textbf{L}inear \textbf{A}dvantage \textbf{S}haping (SLAS) form: $\widetilde{A_i} = \text{sign}(\Delta r) \cdot |\Delta r|^{1+\gamma}$. This design relaxes constraints on high-advantage directions to enable larger updates for informative samples, while tightening constraints on low-advantage, noise-dominated directions to suppress spurious updates under imperfect rewards. Additionally, in practice, removing the prompt-level standard deviation term could introduce sensitivity to reward scaling. We therefore further adopt batch-level standard deviation normalization, which stabilizes training by preventing gradient explosion caused by excessively large advantages as well as inefficient learning under excessively small scales.

Extensive experiments are conducted to evaluate the effectiveness of SLAS. Two backbones of different scales, SD1.4 (0.9B) \citep{rombach2022high} and FLUX.1 Dev (12B) \citep{flux2024} are utilized, with DanceGRPO \citep{xue2025dancegrpo} serving as the baseline under identical reward configurations and training settings, and both methods are trained on the publicly available HPD-v2 training set \citep{wu2023human}. From the training dynamics, SLAS consistently outperforms DanceGRPO, achieving faster and more stable improvements in both HPS-v2.1 and CLIPScore, while DanceGRPO exhibits slower gains and fails to improve the CLIPScore. On the GenEval benchmark \citep{ghosh2023geneval}, SLAS outperforms DanceGRPO, achieving an overall score of 0.53 compared to 0.51 on SD1.4, without degrading basic alignment. Moreover, SLAS demonstrates robustness across model scales: while DanceGRPO drops to 0.54 on FLUX.1 Dev due to reward hacking, SLAS avoids this degradation and improves the overall score to 0.67. On the UniGenBench++ benchmark \citep{wang2025unigenbench++}, SLAS consistently outperforms DanceGRPO: in the short prompt setting, it achieves an overall score of 62.60 over DanceGRPO’s 60.64, while in the more challenging long prompt setting, it further extends this lead with the overall score of 68.32 compared to 65.72. Further qualitative comparisons show that SLAS generates images with stronger semantic coherence, spatial consistency, and compositional reasoning, while effectively mitigating high-frequency artifacts (e.g., over-stylization and exaggerated contrast) commonly observed in DanceGRPO.

Our contributions are summarized as follows:
\contribution{
\begin{enumerate}[leftmargin=*, itemsep=0.2em]
\item We identify that normalization can lead to miscalibration in reinforcement post-training of text-to-image models, and further show that simply removing the prompt-level standard deviation term yields an optimal policy ascent direction that is linear in advantage.
\item We propose Super-Linear Advantage Shaping (SLAS), which is designed to amplify informative, high-advantage directions and suppress noisy, low-advantage updates under imperfect rewards, while ensuring training stability via batch-level normalization.
\item Extensive evaluations demonstrate that SLAS significantly surpasses the DanceGRPO baseline, exhibiting superior out-of-domain performance and robust cross-scenario generalization. 
\end{enumerate}
}

\section{Related Works}
\label{Related Works}
Powering the generation performance of text-to-image models via reinforcement learning has long attracted sustained attention from the research community. Early approaches \citep{blacktraining, fan2023dpok} relied on the PPO \citep{schulman2017proximal} paradigm, but were characterized by high computational overhead and training instability. Furthermore, preference-based methods \citep{wallace2024diffusion, yang2024using} advance this by enabling competitive alignment without the need for an explicit reward model. Recently, Group Relative Policy Optimization (GRPO) \citep{shao2024deepseekmath} has emerged as a scalable substitute. DanceGRPO \citep{xue2025dancegrpo} and Flow-GRPO \citep{liu2025flow} serve as representative studies that extend GRPO to visual generation, unifying diffusion and flow models through SDE reformulation and achieving stable optimization on various generative tasks.

Despite their effectiveness, such approaches can be compromised by illusory advantage signals, leading to reward hacking.  \citet{hong2026understanding} demonstrates that combining multiple reward functions alleviates the problem only to a limited extent, motivating the pursuit of more fundamental improvements. By integrating ratio normalization and gradient reweighting, GRPO-Guard \citep{wang2025grpo} alleviates this issue by moderating the clipping mechanism. Pref-GRPO \citep{wang2025pref} identifies that reward hacking occurs when minimal reward differences between images are exaggerated following normalization, and mitigates this problem through a pairwise preference-based GRPO approach that reformulates the optimization objective from score maximization to preference fitting. GARDO \citep{he2025gardo} incorporates a gated, adaptive regularization mechanism for precise control, coupled with a diversity-aware strategy to encourage comprehensive mode coverage, thereby mitigating the challenge of reward hacking. Further related works can be found in Appendix \ref{Additional Related Work Statement}. 

\section{Preliminaries}
\label{Preliminaries}
Group Relative Policy Optimization (GRPO) \citep{shao2024deepseekmath} employs a group-relative advantage formulation to stabilize policy updates. In the context of text-to-image models, given a group of $G$ generated images $\{\hat{z_0^i}\}_{i=0}^G$, the advantage of the $i$-th output is given by:
\begin{equation}
\hat{A}_t^i = \frac{R(\hat{z_0^i},c) - \text{mean}(R(\hat{z_0^k},c)_{k=1}^G)}{\text{std}(R(\hat{z_0^i},c)_{k=1}^G)}.
\end{equation}
Policy is optimized by maximizing the following objective:
\begin{equation}
\mathcal{J}(\theta)=\mathbb{E}_{c\sim\mathcal{D},\{\hat{z_0^i}\}_{i=1}^G}\frac{1}{G}\sum\limits_{i=1}^G\frac{1}{T}\sum\limits_{t=0}^{T-1}\min\Big(r_t^i(\theta)\hat{A}_t^i, clip(r_t^i(\theta)\hat{A}_t^i), 1-\epsilon, 1+\epsilon\Big) - \beta\mathbb{D}_{\text{KL}}(\pi_{\theta}||\pi_{\text{ref}}),
\end{equation}
where $r_t^i(\theta):= p_{\theta}(x_{t-1}^i|x_{t}^i,c)/p_{\theta_{\text{old}}}(x_{t-1}^i|x_{t}^i,c)$ denotes the importance sampling ratio, and the clip function is defined as $clip(x,l,r):=\max(\min(x,r),l)$.

Furthermore, to meet the exploration requirement, DanceGRPO \citep{xue2025dancegrpo} and Flow-GRPO \citep{liu2025flow} convert the deterministic ODE $dx_t = v_tdt$ to the following equivalent SDE:
\begin{equation}
dx_t = \Big(v_{\theta}(x_t, t)+\frac{\sigma_t^2}{2t}\big(x_t+(1-t)v_{\theta}(x_t,t)\big)\Big)dt + \sigma_td\omega_t,
\end{equation}
where $v_{\theta}(x_t,t)$ is a velocity field trained via flow matching \citep{lipman2022flow}, $d\omega_t$ represents Wiener process increments and $\sigma_t$ controls this stochasticity.  Discretizing it with Euler-Maruyama method yields the following update rule:
\begin{equation}
x_{t+\Delta t} = x_t + \big(v_{\theta}(x_t, t)+\frac{\sigma_t^2}{2t}\big(x_t+(1-t)v_{\theta}(x_t,t)\big)\big)\Delta t+\sigma_t \sqrt{\Delta t}\epsilon, \quad\epsilon\sim\mathcal{N}(0,I).
\end{equation}

\section{Stop Reinforcing Reward Hacking in your Image Generation Model}
\label{Stop Reinforcing Reward Hacking of your Image Generation Model}
\subsection{Problem Setup} Given a prompt $c$, a group of outputs $O=\{\hat{z_0^1},\hat{z_0^2},\cdots,\hat{z_0^G}\}$ is sampled from the policy model $p_{\theta_{\text{old}}}(\cdot|c)$. Let $\alpha\sim\boldsymbol{P}=(p_1,p_2,\cdots,p_n)$ denote true reward probability distribution over the outputs, which satisfies that $\sum_{i=1}^n p_i = 1$. Here, $p_i$ denotes the $i$-th element of the\textbf{ true underlying reward distribution} over outputs, which is unobservable in practice and exists only as a theoretical construct reflecting real-world situations. Setting that $\boldsymbol{\hat{P}} = (\hat{p_1}, \hat{p_2},\cdots, \hat{p_n})$ denotes the predicted reward probability distribution over the outputs that are generated by the policy model $p_{\theta_{\text{old}}}(\cdot|c)$. Here, $\hat{p_i}$ denotes the $i$-th component of the \textbf{empirical predicted reward distribution}, estimated through large-scale sampling. For other aspects, we adopt the same notation with DanceGRPO \citep{xue2025dancegrpo}. It is worth mentioning that the reward function assigns a nonzero reward $r(c, \hat{z_0};\alpha)$ only at $t=0$, while being zero at all other time steps. Here, $\alpha$ is introduced to represent the reward signal from a distributional perspective, as opposed to a conventional point estimate. Consider the sum of advantages over all $T$ steps, we denote that:
\begin{equation}
\sum \limits_{t = 1}^{T}A^{p_{\theta_{\text{old}}}}(c,t,\hat{z_t}, \hat{z}_{t-1})  \xlongequal{} A^{p_{\theta_{\text{old}}}}(c, \boldsymbol{\hat{P}}).
\end{equation}
Furthermore, we set that:
\begin{equation}
\boldsymbol{\eta}=(\eta_1, \eta_2, \cdots, \eta_n)=\mathbb{E}_{t\sim\mathcal{U}(1,T), \hat{z_t}'\sim\boldsymbol{\hat{P'}}} r(c, \hat{z_t};\alpha).
\end{equation}
Here, $\eta_i$ represents the expected reward associated with the $i$-th discrete histogram bin, averaged over all timesteps and rollouts generated by the policy model. Specifically, histogram binning discretizes the distribution into $n$ non-overlapping bins, each corresponding to a reward interval. This yields a discrete probability distribution, where each bin carries a probability mass indicating the likelihood that an output’s reward falls within it. Therefore, we have that:
\begin{equation}
\label{true advantage}
A^{p_{\theta_{\text{old}}}}(c, \boldsymbol{\hat{P}}) = \sum \limits_{k = 1}^{n}p_k\cdot \big(r(c, z_0;\alpha, p_k)-\eta_k\big).
\end{equation}
Here, $A^{p_{\theta_{\text{old}}}}(c, \boldsymbol{\hat{P}})$ denotes the true advantage for prompt $c$ and predicted reward distribution $\boldsymbol{\hat{P}}$ under the policy model. 

\subsection{Normalization Leads to Miscalibration}
\label{Normalization Leads to Miscalibration}
Building upon the above formulation and drawing inspiration from \citep{bereket2025uncalibrated}, we derive \Cref{standard_GRPO}. A detailed proof is provided in Appendix~\ref{Proof of Theorem 1}.
\begin{proposition}
\label{standard_GRPO}
Suppose the group number $G$ is large enough, and we set that:
\[
\sigma_k = \mathbb{E}_{(\hat{z_0^1}, \hat{z_0^2},\cdots,\hat{z_0^G})\sim p_{\theta_{\text{old}}}}\bigg[\text{std}\Big(r(c,\hat{z_0^1}; \alpha, p_k)_1, r(c,\hat{z_0^2};\alpha, p_k)_2, \cdots, r(c,\hat{z_0^G};\alpha, p_k)_G\Big)\bigg]
\]
denotes the expected reward standard deviation under the $k$-th discrete reward histogram. Then, the expected standardized GRPO advantage satisfies that:
\begin{equation}
\mathbb{E}_{\alpha\sim \boldsymbol{P}, (\hat{z_0^2},\cdots,\hat{z_0^G})\sim p_{\theta_{\text{old}}}(\cdot|c)}\Big[\hat{A}_{\text{std}}^{p_{\theta_{\text{old}}}}(c, T, \hat{z_0^1})\Big]
= \frac{G-1}{G} \cdot \sum \limits_{j = 1}^{n}\frac{1}{\sigma_j + \epsilon}\cdot p_j \cdot (r(c,z_0;\alpha,p_j)-\eta_j).
\end{equation}
\end{proposition}

This proposition reveals that the resulting advantage constitutes \textbf{a biased estimator} of the true advantage. Beyond the factor $(G-1)/G$, the estimator further incorporates a multiplicative inverse standard deviation term, which is expressed as $1/(\sigma_j+\epsilon)$. This could lead to a pronounced overestimation of the advantage. We provide further discussion on this in Appendix~\ref{Illustrative Examples of Low-Variance Property}.  Furthermore, we identify three key avenues for improvement: data, reward, and algorithm:
\contribution{
\begin{enumerate}[leftmargin=2.5em, itemsep=0em]
    \item[\textbf{\textit{\textcolor{myblue}{Imp1:}}}]
    Prompts whose sampled outcomes exhibit low reward variance should be filtered out.
    \item[\textbf{\textit{\textcolor{myblue}{Imp2:}}}]
    More advanced reward models with stronger discriminative power could be better.
    \item[\textbf{\textit{\textcolor{myblue}{Imp3:}}}]
   Adjusting the prompt-level standard deviation term could reduce the degree of advantage overestimation, thereby mitigating the risk of reward hacking.
\end{enumerate}
}
In this work, we focus on the third aspect above, specifically from an algorithmic perspective. Firstly, we study the variant of removing the prompt-level standard deviation term. With this adjustment, the advantage estimation can be expressed as:
\begin{equation}
\label{changed_advantage}
A_i = r_i-\text{mean}({r_1, r_2, \cdots, r_G}).
\end{equation}

\subsection{Optimal Policy Ascent is Linear in Advantage}       
Let $\mathcal{X}$ denotes the input space and $\mathcal{Y}$ denotes the output space. Consider the conditional probability measure space:
\[
\Pi : =\{\pi(y|x)|\pi(y|x)\in\mathcal{P}(\mathcal{Y}),\,\forall x\in\mathcal{X}\}. 
\]
Consider a conditional policy $\pi_{\theta}(y|x)\in\mathcal{P}(\mathcal{Y})$ as an element in a probability simplex over the output space $\mathcal{Y}$, parameterized implicitly by $\theta$. Therefore, we can define the expected return functional as:
\begin{equation}
\boldsymbol{\mathcal{F}}(\pi(y|x))=\mathbb{E}_{y\sim\pi(y|x)}[r(x,y)].
\end{equation}
Furthermore, the policy optimization can be cast as a constrained functional optimization problem:
\begin{equation}
\label{constrained functional optimization problem}
\max\limits_{\pi(y|x)} \boldsymbol{\mathcal{F}}(\pi(y|x)) \quad \text{s.t.} \;\mathbb{D}_{\text{KL}}(\pi(y|x)||\pi_{0}(y|x))\leq \delta.
\end{equation}
And this can be reformulated as the following objective:
\begin{equation}
\label{objective}
\max\limits_{\pi(y|x)}\{ \boldsymbol{\mathcal{F}}(\pi(y|x))-\frac{1}{\eta}\cdot\mathbb{D}_{\text{KL}}(\pi(y|x)||\pi_0(y|x)) \}.
\end{equation}
\begin{definition}
\label{admissible variations}
Let $\delta\pi(y|x)$ be an admissible variation, i.e., a signed measure satisfying:
\begin{equation}
\label{constrain}
\forall x, \int \delta \pi(y|x) dy = 0.
\end{equation}
\end{definition}
Furthermore, we characterize the functional form of the optimal ascent direction of the policy $\pi(y|x)$. Applying a Taylor expansion, the local functional update objective can be expressed as:
\begin{equation}
\label{functional local update problem}
\max\limits_{\delta\pi(y|x)}\{\delta \boldsymbol{\mathcal{F}}(\pi_0(y|x); \delta\pi(y|x)) - \frac{1}{2\eta}\int_{\mathcal{Y}}\frac{(\delta\pi(y|x))^2 }{\pi_0(y|x)} dy\}
\end{equation}
The solution to this problem yields the optimal update, which is summarized in \Cref{optimal ascent}. A detailed proof is provided in Appendix~\ref{Proof of Theorem 2}.
\setcounter{theorem}{1}
\begin{theorem}
\label{optimal ascent}
Considering all perturbations $\delta\pi(y|x)$ that satisfy the constraint in \Cref{constrain}, and noting that the functional $\boldsymbol{\mathcal{F}}(\pi(y|x))$ is defined on the measure space, we assume that the policy $\pi(y|x)$ is absolutely continuous. Then, the optimal ascent is characterized by the following result:
\begin{equation}
\delta\pi(y|x)\propto\pi_0(y|x) \cdot A(x, y),
\end{equation}
where $\pi_0(y|x)$ denotes the current policy function, $A(x, y) = r(x,y) + \lambda(x)$ represents the advantage function class.
\end{theorem}
Noting that $A(x, y) = r(x,y) + \lambda(x)$, where $\lambda(x)$ serves as a baseline function, we recover a formulation consistent with \Cref{changed_advantage}. Furthermore, \Cref{optimal ascent} implies that the optimal ascent direction is proportional to the advantage function $A(x,y)$, establishing a linear relationship.

\section{Super-Linear Advantage Shaping}
\label{Super-Linear Advantage Shaping}
Although removing the prompt-level standard deviation term alleviates miscalibration, the resulting linearity limits its ability to distinguish between genuine advantages and noise-dominated ones when the reward function is imperfect. To this end, we revisit the functional local update problem in \Cref{functional local update problem} from an information geometry perspective \citep{ay2017information}. Specifically, the quadratic functional optimization is carried out on the tangent space of the probability simplex, with the tangent space endowed with the following metric:
\begin{equation}
\big<\delta\pi_1(y|x),\delta\pi_2(y|x)\big>_{\pi_0(y|x)} = \int_{\mathcal{Y}}\frac{\delta\pi_1(y|x)\delta\pi_2(y|x)}{\pi_0(y|x)} dy,
\end{equation}
where perturbations $\delta\pi_1(y|x)$ and $\delta\pi_2(y|x)$, also referred to as admissible variations, lie in the tangent space of the probability simplex. Indeed, this is precisely the Fisher–Rao information metric.

In text-to-image tasks, due to inherent limitations of the reward model, samples with low advantage contribute little to training gains and are more prone to reward hacking. Conversely, high-advantage samples yield larger and safer training improvements. Therefore, in the tangent space of the probability simplex, we can consider imposing more relaxed constraints on those samples that offer higher and safer training gains. Motivated by this, we define the following $\gamma$-weighted variational metric:
\begin{equation}
\label{gamma-weighted variational metric}
\big<\delta\pi_1(y|x),\delta\pi_2(y|x)\big>_{\pi_0(y|x)}^\gamma= \int_{\mathcal{Y}}\frac{\delta\pi_1(y|x)\delta\pi_2(y|x)}{\Phi_\gamma(|A(x,y)|)\pi_0(y|x)} dy,
\end{equation}
where $\Phi_\gamma(\cdot)$ is a monotone non-decreasing mapping. Among the valid formulations, we adopt the power-law form $\Phi_\gamma(|A(x, y)|) = |A(x, y)|^\gamma$ for simplicity. Such a paradigm reshapes the geometry of the policy space itself by assigning advantage-dependent curvature to the probability simplex. As a result, the functional local update problem in \Cref{functional local update problem} can be further expressed as:
\begin{equation}
\label{monotone non-decreasing mapping}
\max\limits_{\delta\pi(y|x)}\{\delta \boldsymbol{\mathcal{F}}(\pi_0(y|x); \delta\pi(y|x)) - \frac{1}{2\eta}\int_{\mathcal{Y}}\frac{(\delta\pi(y|x))^2 }{|A(x, y)|^\gamma\cdot\pi_0(y|x)} dy\}.
\end{equation}
In a similar manner, we can solve the resulting optimization problem. The results are presented in \Cref{optimal_ascent_super}, with the corresponding proofs provided in Appendix \ref{Proof of Theorem 3}.
\begin{theorem}
\label{optimal_ascent_super}
Under the same conditions as in \Cref{optimal ascent}, after transforming to the tangent space of the probability simplex as expressed in \Cref{gamma-weighted variational metric}, the optimal ascent is characterized by:
\begin{equation}
\delta\pi(y|x)\propto\pi_0(y|x) \cdot \text{sign}(A(x, y))|A(x,y)|^{1+\gamma},
\end{equation}
where $\pi_0(y|x)$ denotes the current policy function, $\text{sign}(\cdot)$ denotes the sign function, $A(x, y)$ represents the advantage function class.
\end{theorem}
\begin{wrapfigure}{r}{0.55\linewidth}
    \centering
    \includegraphics[width=1.0\linewidth]{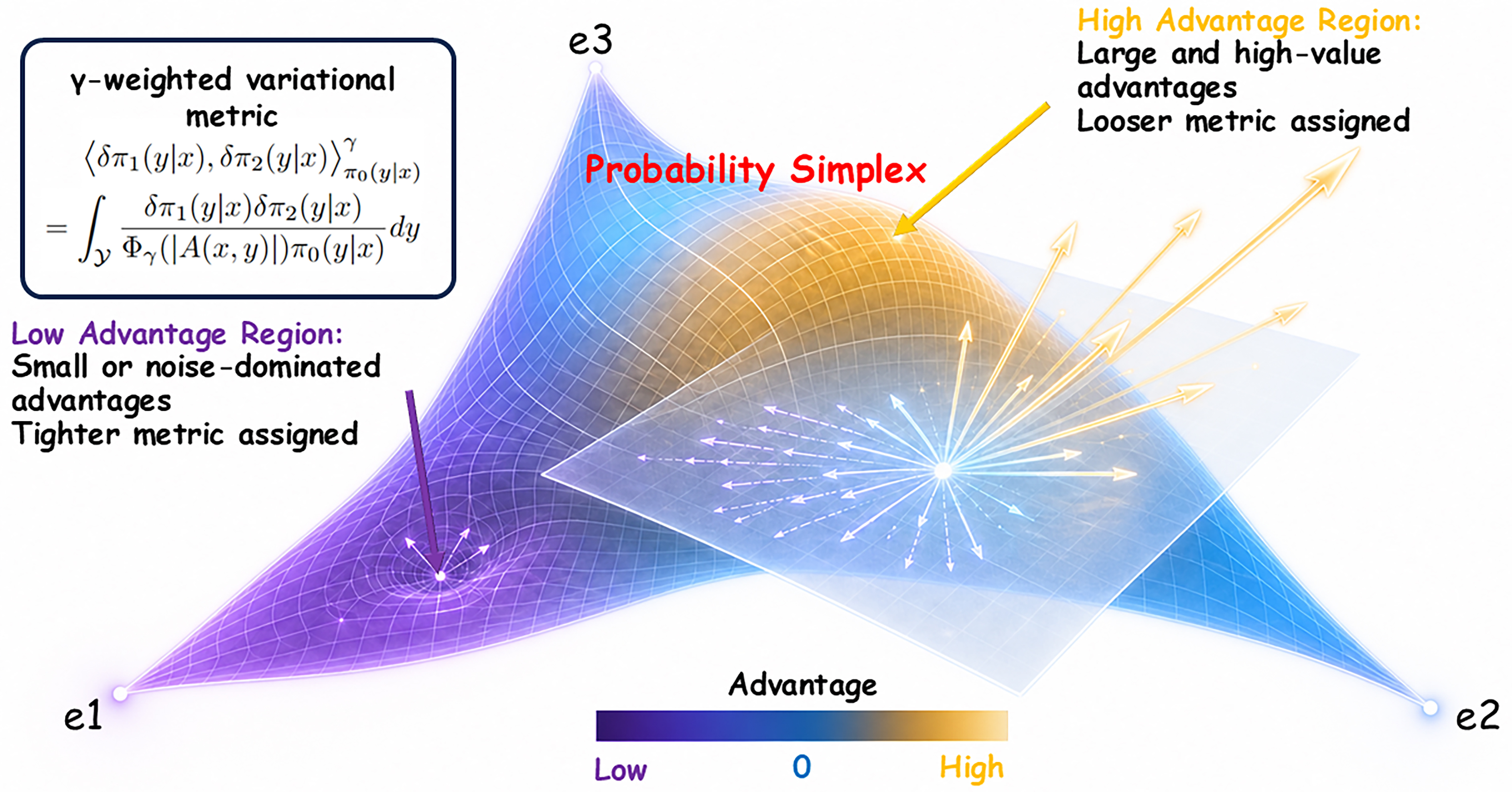}
    \caption{Super-Linear Advantage Shaping. $\gamma$-weighted variational metric $\Phi_{\gamma}|A(x,y)|$ reshapes local geometry of the probability simplex, amplifying high-advantage directions while suppressing noisy, low-advantage ones.}
    \label{NeurIPS_2026_mainx}
\end{wrapfigure}
From \Cref{optimal_ascent_super}, it can be concluded that introducing advantage dependent weighting $\Phi_\gamma(|A(x,y)|)$ into vanilla Fisher–Rao information metric is equivalent to an advantage-dependent rescaling in the tangent space. To be specific, those associated with large advantages are endowed with a looser metric, allowing larger updates; conversely, those associated with small or noise-dominated advantages are assigned a tighter metric, thereby mitigating excessive updates induced by weak or noisy advantage signals. This super-linear transformation serves as an advantage-based preconditioning, steering updates toward samples of higher informational value in the information-geometric sense.

Moreover, from an optimization standpoint, the update can be seen as executing steepest ascent under a task-adaptive Riemannian metric. Instead of directly modifying the reward or advantage function, the proposed formulation reshapes the underlying geometry of the space, leading to a super-linear form of advantage shaping. From this perspective, super-linear advantage shaping enhances robustness under imperfect reward models by attenuating the influence of low-confidence samples while permitting larger updates along directions supported by strong advantage signals. Accordingly, we define the practically utilized advantage function in \Cref{changed_advantage} after super-linear shaping as:
\begin{equation}
\label{reshaped super-linear advantage}
\widetilde{A_i} = \text{sign}(\Delta r) \cdot |\Delta r|^{1+\gamma},
\end{equation}
where $\Delta r = r_i-\text{mean}({r_1, r_2, \cdots, r_G})$ represents the vanilla linear advantage. Further discussion of $\gamma$, specifically its trust-region bounds, is provided in Appendix~\ref{Trust-Region Bounds on gamma}. 

Moreover, simply removing the standard deviation term introduces the sensitivity to reward scaling. When the reward scale shifts, applying only mean centering causes the scale of the advantage function to change accordingly. An excessively large advantage scale may lead to gradient explosion, while an overly small scale can degrade learning efficiency. From this perspective, standard deviation normalization mitigates scale sensitivity and stabilizes training. Therefore, considering these factors, we adopt \textbf{batch-level standard deviation normalization} in place of its prompt-level counterpart.

\section{Experiments}
\label{Experiment}

\setlength{\tabcolsep}{1mm}
\begin{table*}[t]
\centering
\caption{Quantitative Comparison on the GenEval Benchmark. Obj.: Object. Attr.: Attribution.}
\small
\begin{threeparttable}
\begin{adjustbox}{max width=\textwidth}
\label{geneval}
\begin{tabular}{cccccccc}
\toprule
\textbf{Model \& Method} & \textbf{Single Obj.} & \textbf{Two Obj.} & \textbf{Counting} & \textbf{Colors} & \textbf{Position} & \textbf{Color Attr.} &\textbf{Overall}$\uparrow$ \\
\midrule
\multicolumn{8}{c}{\textbf{Baseline and Our Method}} \\
SD1.4 \citep{rombach2022high} & 0.98 & 0.36 & 0.35 & 0.73 & 0.01 & 0.07 & 0.42 \\
SD1.4 + DanceGRPO & 0.98 & 0.67 & 0.44 & 0.77 & 0.11 & 0.10 & 0.51 \\
SD1.4 + SLAS &\cellcolor{myblue!50!white}\textbf{0.98} & \cellcolor{myblue!50!white}\textbf{0.71} & \cellcolor{myblue!50!white}\textbf{0.42} & \cellcolor{myblue!50!white}\textbf{0.82} & \cellcolor{myblue!50!white}\textbf{0.11} & \cellcolor{myblue!50!white}\textbf{0.15} & \cellcolor{myblue!50!white}\textbf{0.53} \\ 
FLUX.1 Dev \citep{flux2024} & 0.98 & 0.81 & 0.74 & 0.79 & 0.22 & 0.45 & 0.66 \\
FLUX.1 Dev + DanceGRPO & 0.88 & 0.65 & 0.52 & 0.62 & 0.24 & 0.35 & 0.54 \\
FLUX.1 Dev + SLAS & \cellcolor{myblue!50!white}\textbf{1.00} & \cellcolor{myblue!50!white}\textbf{0.86} & \cellcolor{myblue!50!white}\textbf{0.64} & \cellcolor{myblue!50!white}\textbf{0.77} & \cellcolor{myblue!50!white}\textbf{0.23} & \cellcolor{myblue!50!white}\textbf{0.49} & \cellcolor{myblue!50!white}\textbf{0.67} \\
\bottomrule
\end{tabular}
\end{adjustbox}
\label{GenEval_results}
\end{threeparttable}
\end{table*}

\begin{figure}
\centering
\includegraphics[width=1.0\linewidth]{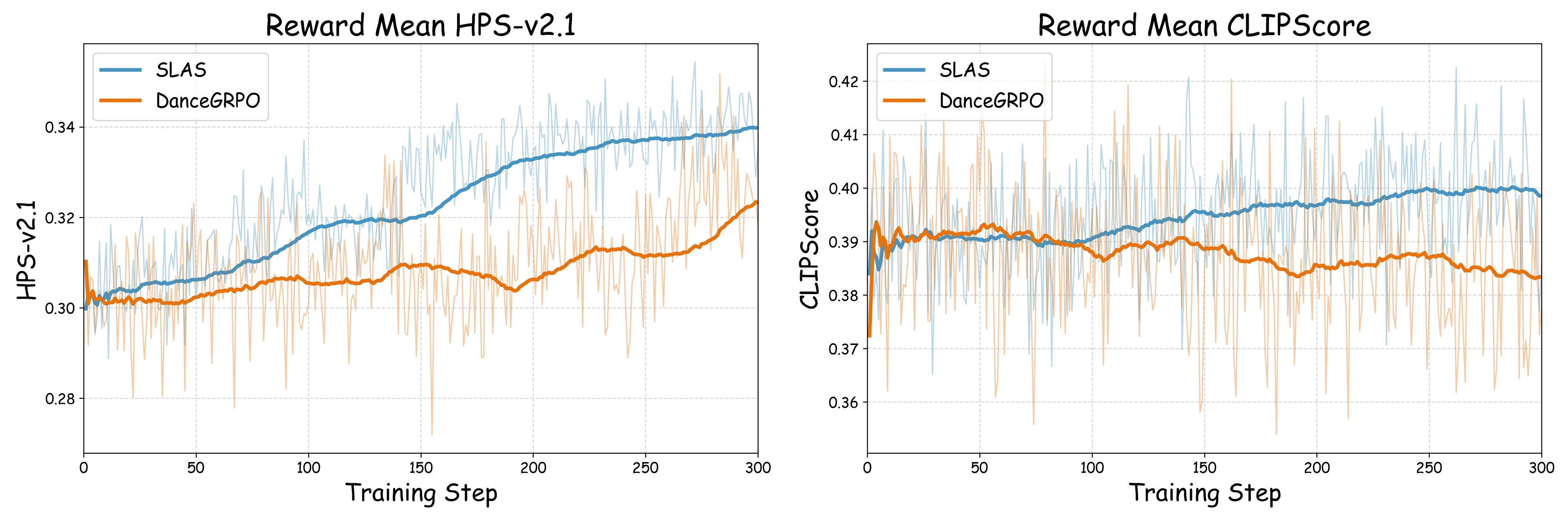}
\caption{Training Dynamics on FLUX.1 Dev. \textit{Left:} Reward Mean of HPS-v2.1. \textit{Right:} Reward Mean of CLIPScore. \textcolor{myblue}{Blue lines} represent the \textcolor{myblue}{SLAS}, while \textcolor{myorange}{orange lines} denote the \textcolor{myorange}{DanceGRPO}.}
\label{Reward_Plot}
\end{figure}

\subsection{Experimental Setup}
\label{Experimental Setup}
\paragraph{Implementation Details:} Throughout our experiments, we adopt DanceGRPO \citep{xue2025dancegrpo} as the baseline for comparison. Two models with different parameter scales, Stable Diffusion v1.4 (0.9B) \citep{rombach2022high} and FLUX.1 Dev (12B) \citep{flux2024}, are selected as backbones to validate the effectiveness of our method. The training reward is defined as a combination of scores from CLIP \citep{radford2021learning} and HPS-v2.1 \citep{wu2023human}, consistent with the configuration used in DanceGRPO \footnote{\label{github}https://github.com/XueZeyue/DanceGRPO/issues/19}. Specifically, the HPS-v2.1 and CLIP reward components are weighted at a ratio of 1:1 for SD1.4, and 0.7:1.4 for FLUX.1 Dev. Regarding training data, the original DanceGRPO paper uses a carefully curated internal prompt set; in contrast, to ensure a fair comparison and to assess the robustness of our method, we employ the publicly available HPD-v2 training set \citep{wu2023human}. Experiments for FLUX.1 Dev are conducted on two nodes, each equipped with eight NVIDIA A800 GPUs, while those for SD1.4 are conducted on a single node. All other experimental settings are consistent with those of DanceGRPO.
\paragraph{Evaluation Details:} Our evaluation primarily focuses on assessing the model's out-of-domain performance and cross-scenario generalization capability. Firstly, we conduct evaluation using the GenEval \citep{ghosh2023geneval}, an open-source object-centric framework. This framework enables fine-grained, instance-level text-image alignment assessment, going beyond conventional holistic metrics to verify compositional properties for text-to-image generative models. Moreover, to achieve a more fine-grained and multi-faceted evaluation, we adopt the UniGenBench++ \citep{wang2025unigenbench++}. It organizes evaluation around 10 primary dimensions covering diverse semantic requirements, enabling a comprehensive assessment of the model generalization capability.

\begin{table}[t]
\centering
\caption{Quantitative Comparison on the UniGenBench++ Benchmark; Gemini 2.5 Pro \citep{comanici2025gemini} is used as the MLLM for evaluation. We report primary results of the \textcolor{myblue}{English Short Prompt Evaluation} and the \textcolor{myorange}{English Long Prompt Evaluation}. The detailed results are presented in Appendix~\ref{Detailed Results on the UniGenBench++ benchmark}.}
\resizebox{\linewidth}{!}{
\begin{tabular}{l|c|c|c|c|c|c|c|c|c|c|c}
\toprule
\rowcolor{myblue!75!white}\multicolumn{12}{c}{\textbf{English Short Prompt Evaluation--Primary}} \\
\midrule
\textbf{Method} & \textbf{Overall} & \textbf{Style} & \textbf{World Know.} & \textbf{Attribute} & \textbf{Action} & \textbf{Relation} & \textbf{Logic.Reason.} & \textbf{Grammar} & \textbf{Compound} & \textbf{Layout} & \textbf{Text} \\
\midrule
FLUX.1 Dev & 61.08 &86.20 & 87.66 & 63.78 & 61.69 & 65.99 & 27.52 & 60.96 & 47.29 & 71.08 & 37.64\\
DanceGRPO& 60.64 & 76.30 & 85.28 & 65.17 & 61.69 & 66.62 & 37.61 & 56.95 & 57.73 & 70.90 & 28.16\\
SLAS& 62.60 & 82.30 & 81.96 & 67.84 & 70.15 & 69.16 & 41.06 & 54.41 & 57.22 & 73.69 & 28.16\\ \midrule
\rowcolor{myorange!75!white}\multicolumn{12}{c}{\textbf{English Long Prompt Evaluation--Primary}} \\
\midrule
\textbf{Method} & \textbf{Overall} & \textbf{Style} & \textbf{World Know.} & \textbf{Attribute} & \textbf{Action} & \textbf{Relation} & \textbf{Logic.Reason.} & \textbf{Grammar} & \textbf{Compound} & \textbf{Layout} & \textbf{Text} \\
\midrule
FLUX.1 Dev& 69.22 & 90.37  & 89.45  & 80.18  & 65.30  & 68.64  & 47.06  & 70.15  & 67.06  & 78.41  & 35.60  \\
DanceGRPO& 65.72 & 79.24  & 87.28  & 77.20  & 65.34  & 67.40  & 49.02  & 63.39  & 67.57  & 76.90  & 23.91   \\
SLAS& 68.32  & 83.55  & 87.43  & 79.20  & 67.35  & 69.45  & 53.19  &63.78  & 69.39  & 79.44  & 30.43   \\ \midrule
\end{tabular}
}
\vspace{-1em}
\label{UniGenBench++ BenchmarK}
\end{table}
\begin{figure*}[t]
\centering
\includegraphics[width=1.0\linewidth]{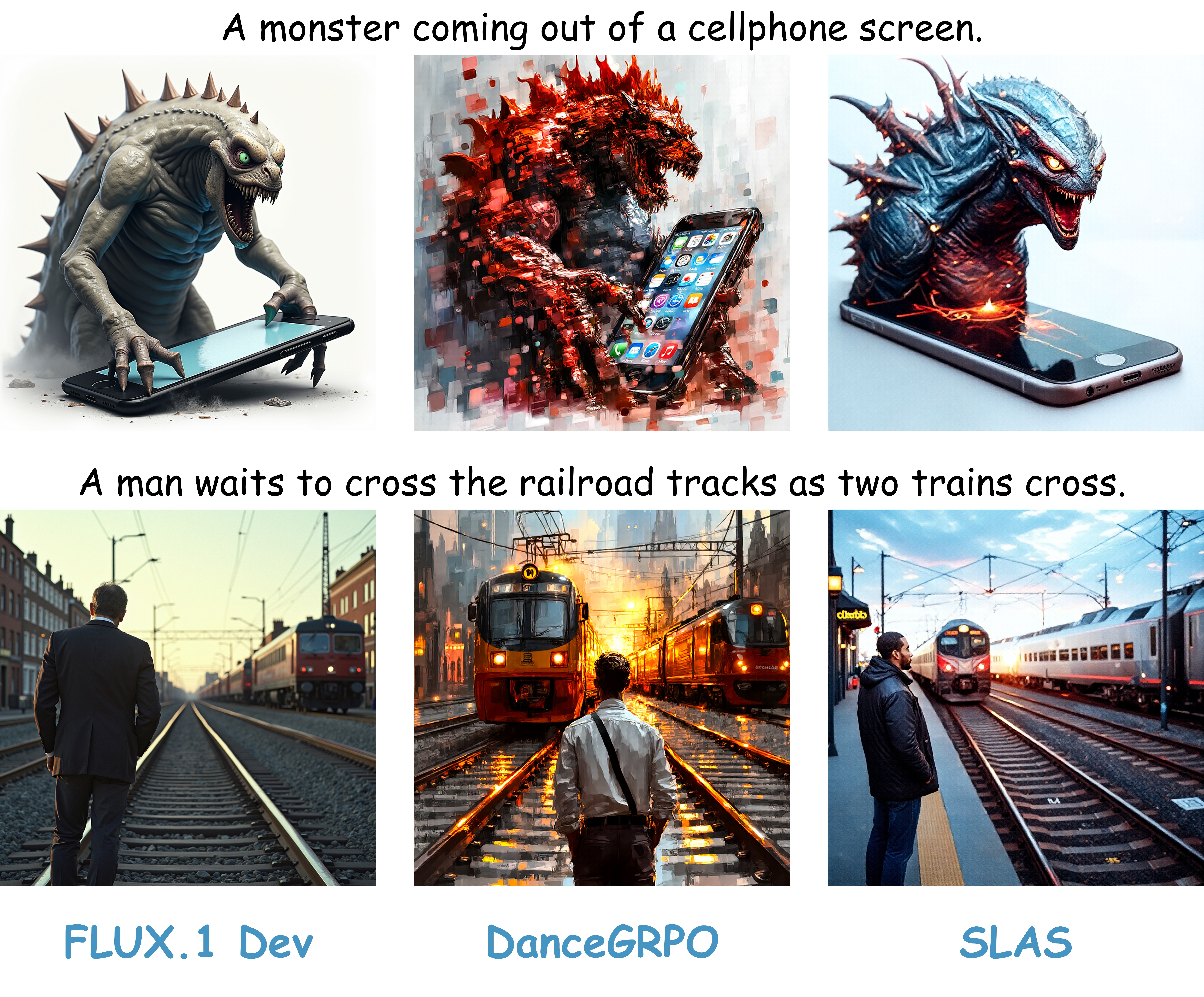}
\vspace{-2em}
\caption{Qualitative comparison of generations from FLUX.1 Dev, DanceGRPO, and SLAS.}
\vspace{-1em}
\label{visual_main}
\end{figure*}

\subsection{Quantitative Comparison}
\label{Quantitative Comparison}
\paragraph{Training Dynamics.}  As shown in \Cref{Reward_Plot}, SLAS demonstrates superior in-domain performance compared to DanceGRPO on both HPS-v2.1 and CLIPScore. In the left panel, we present the HPS-v2.1 (human preference) reward. After a comparable start, SLAS rapidly surpasses DanceGRPO and steadily improves to a higher final score, whereas DanceGRPO exhibits slower progress followed by a late-stage increase. In the right panel, we present the CLIPScore (instruction following) reward. DanceGRPO fails to achieve meaningful gains throughout optimization, whereas SLAS maintains a robust and sustained upward trajectory.

\paragraph{GenEval Benchmark.}  \Cref{GenEval_results} demonstrates the GenEval benchmark results of baseline and our method. On SD1.4, SLAS achieves the best overall performance, lifting the score from 0.42 to 0.53, surpassing the 0.51 achieved by DanceGRPO. Gains are most pronounced on compositional categories such as Two-Object, Colors, and Color Attribution, showing stronger sensitivity to structured and attribute-level semantics rather than uniform score inflation. Meanwhile, simple attribute performance is preserved, with no degradation to basic alignment. On the stronger FLUX.1 Dev backbone, SLAS remains robust, yielding a consistent improvement from 0.66 to 0.67, whereas vanilla DanceGRPO exhibits severe degradation, dropping to 0.54. This contrast highlights the risk of miscalibrated advantage scaling in standard GRPO, leading to reward hacking and poor generalization on stronger models. SLAS, however, maintains robust performance improvement across model scales.

\paragraph{UniGenBench++ Benchmark.} \Cref{UniGenBench++ BenchmarK} presents the UniGenBench++ benchmark results of baseline and our method. Across both settings, SLAS consistently outperforms DanceGRPO. In the short prompt setting, SLAS achieves an overall score of 62.60, surpassing DanceGRPO’s 60.64, while delivering substantial gains in logical reasoning, action understanding, and relation modeling. In the more challenging long prompt setting, SLAS achieves an overall score of 68.32, surpassing DanceGRPO’s 65.72, and consistently outperforms it across all sub-categories. Additionally, it is worth noting that in the long prompt setting, both DanceGRPO and SLAS achieve lower overall scores compared to FLUX.1 Dev. This gap is likely attributable to a bias toward short prompts in the HPD-v2 training data, further highlighting that a balanced prompt-length distribution in training data could be crucial for industrial-scale model deployment.

\subsection{Qualitative Comparison}
\label{Qualitative Comparison}
In \Cref{visual_main}, we present a qualitative comparison between SLAS and DanceGRPO. Overall, SLAS exhibits stronger semantic coherence, spatial consistency, and compositional reasoning abilities. For the prompt ``A monster coming out of a cellphone screen'', FLUX.1 Dev fails to capture the core semantics of ``emerging from the screen''while DanceGRPO introduces disruptive pixelated over-stylization; in contrast, SLAS accurately depicts the creature breaking through the display with natural details. For the more complex compositional prompt ``A man waits to cross the railroad tracks as two trains cross'', FLUX.1 Dev omits the critical constraint of ``two trains'', and DanceGRPO misinterprets the trains' spatial relationship with over-exaggerated contrast; in comparison, SLAS faithfully captures all semantic elements and preserves a logically consistent scene layout. Meanwhile, SLAS tends to prioritize clear and high-margin improvements over marginal gains, thereby alleviating the high-frequency artifacts, alleviating the high-frequency artifacts (e.g., over-stylization and exaggerated contrast) observed in DanceGRPO. This further mitigates semantic drift under complex prompts and leads to better alignment with human preference as well as more faithful instruction following. We provide additional qualitative examples in Appendix~\ref{Additional Qualitative Examples}. Furthermore, we also apply the method to product-level e-commerce image editing, with qualitative examples in Appendix~\ref{Product-level E-commerce Image Editing Application}.

\section{Conclusion, Limitation and Future Work}
\label{Conclusion}
In this paper, we first analyze and demonstrate that prompt-level normalization in GRPO-style advantage computation induces miscalibration, thereby facilitating reward hacking in reinforcement post-training of text-to-image models. Furthermore, we consider directly removing the prompt-level standard deviation term; however, this yields an optimal policy ascent direction that is linear in advantage. By introducing an advantage-weighted Fisher–Rao geometry that nonlinearly reshapes the policy space, we propose Super-Linear Advantage Shaping (SLAS). This mechanism amplifies informative directions while suppressing noisy updates, thereby enabling more effective policy optimization.
Extensive evaluations demonstrate that SLAS significantly surpasses the DanceGRPO baseline, exhibiting superior out-of-domain performance and robust cross-scenario generalization. \paragraph{Limitation and Future Work.} In this work, we primarily focus on the algorithmic perspective, leaving the other two dimensions (data and reward) for future research. First, from a data standpoint, our study is restricted to the public HPD-v2 dataset, and the question of ``\textit{what types of prompts are most suitable for reinforcement learning}'' remains an open problem. Second, regarding the reward mechanism, our study adopts a blended metric of HPS-v2.1 and CLIPScore, and the challenge of ``\textit{how to optimally select or formulate the most effective reward function}'' warrants further exploration.

\clearpage
\bibliographystyle{unsrtnat}  
\bibliography{main}

\newpage
\appendix
\setcounter{theorem}{1}
\setcounter{proposition}{0}
\section{Illustrative Examples of Low-Variance Property}
\label{Illustrative Examples of Low-Variance Property}
\begin{figure}[h]
    \centering
    \includegraphics[width=1.0\linewidth]{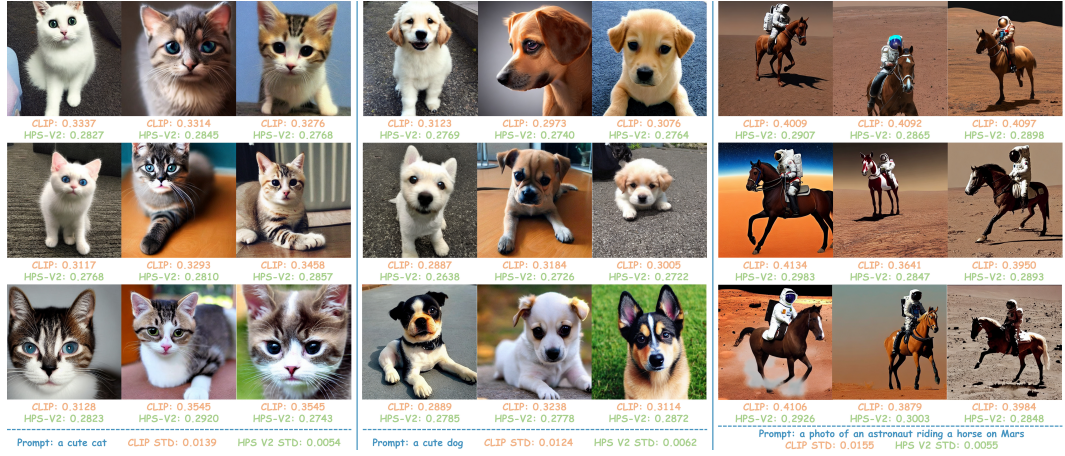}
    \caption{Illustrative Examples: Different Generations under the Same Prompt. \textcolor[HTML]{4594c1}{Blue} represents the \textcolor[HTML]{4594c1}{Prompt}. \textcolor[RGB]{244,177,131}{Orange} represents the \textcolor[RGB]{244,177,131}{CLIPScore}. \textcolor[RGB]{169,209,142}{Green} represents the \textcolor[RGB]{169,209,142}{HPS-v2.1}. Zoom in for more details.}
    \label{illustrative examples}
\end{figure}
The most widely adopted reward models for text-to-image tasks, including CLIP \citep{radford2021learning} and HPS-v2.1 \citep{wu2023human}, share a critical property: the reward variance for samples generated from the same prompt is extremely small. These pre-trained reward models tend to assign highly concentrated scores to semantically similar images, and fail to distinguish subtle differences, which results in an extremely compact reward distribution across the group of samples. We present this property with examples in \Cref{illustrative examples}. Combined with our theoretical conclusion, the normalized advantage assigns a weight of $1/(\sigma_j+\epsilon)$ to each reward bin. Although $\sigma_j$ is defined theoretically, the above examples indicate that it tends to be small, which can result in extreme advantage amplification. This over-amplification directly causes the model to devote the majority update effort to optimizing spurious tiny differences: the model will continuously adjust generation outputs to squeeze out every minor improvement in the reward model's scores. This eventually evolves into the typical reward hacking behavior: the proxy score of the reward model keeps increasing, while the real image quality deteriorate continuously.
\section{Proof of Proposition 1}
\label{Proof of Theorem 1}
Prior to the formal proof of Proposition 1, we first establish the validity of the assumption that ``the group number $G$ is large enough'' in \Cref{lemma_assumption_1} and \Cref{lemma_assumption_2}.
\begin{lemma}
\label{lemma_assumption_1}
Consider a group of G samples, and define the sample standard deviation as $s_G=\sqrt{\frac{1}{G-1}\sum_{i=1}^G(r_i-\bar r)^2}$, where $\bar r$ denotes the sample mean. Assuming the reward distribution has a finite fourth-order moment $\mu_4$, then the deviation between $s_G$ and $\sigma$ satisfies that:
\begin{equation}
\mathbb{E}[s_G]-\sigma=-\frac{1}{8\sigma^3}\frac{\mu_4}{G}+\frac{G-3}{8G(G-1)}\sigma+o_p\left(\frac{1}{G}\right)
\end{equation}
\end{lemma}
\begin{proof}
By the Central Limit Theorem, we can obtain that: $\sqrt{G}(s_G^2-\sigma^2)\rightarrow N(0, \mu_4-\sigma^4)$. 

Perform a Taylor expansion of the function $f(x)=\sqrt{x}$ around $x=\sigma^2$: 
\begin{equation}
f(s_G^2)=f(\sigma^2)+f'(\sigma^2)(s_G^2-\sigma^2)+\frac{1}{2}f''(\sigma^2)(s_G^2-\sigma^2)^2 + \frac{1}{6}f'''(\xi)(s_G^2-\sigma^2)^3
\end{equation}
Furthermore, notice that:
\begin{equation}
\mathbb{E}[(s_G^2-\sigma^2)^2]=\text{Var}(s_G^2)=\frac{\mu_4-\frac{G-3}{G-1}\sigma^4}{G} \quad \text{and} \quad \mathbb{E}[(s_G^2-\sigma^2)]=0.
\end{equation}
Then, substitute the derivatives and take the expectation on both sides. Finally, we can obtain that: 
\begin{equation}
\mathbb{E}[s_G]=\sigma+\frac{1}{2\sigma}\mathbb{E}[s_G^2-\sigma^2]-\frac{1}{8\sigma^3}\mathbb{E}[(s_G^2-\sigma^2)^2]+o_p(\frac{1}{G})=\sigma-\frac{1}{8\sigma^3}\frac{\mu_4}{G}+\frac{G-3}{8G(G-1)}\sigma+o_p(\frac{1}{G}),
\end{equation}
which completes the proof.
\end{proof}
\begin{lemma}
\label{lemma_assumption_2}
If reward follows the Gaussian distribution, we have that: 
\begin{equation}
\mathbb{E}[s_G]=\sigma \cdot \sqrt{\frac{2}{G-1}} \cdot \frac{\Gamma(\frac{G}{2})}{\Gamma(\frac{G-1}{2})}.
\end{equation}
Then, the relative bias ratio, defined as $R_{\text{RelBias}}=(\mathbb{E}[s_G]-\sigma)/\sigma$ equals $-0.035$ for $G=8$ and $-0.017$ for $G=16$.
\end{lemma}
\begin{proof}
The standardized sample variance follows a chi-squared distribution:
\begin{equation}
\frac{(G-1)s_G^2}{\sigma^2}\sim\chi^2(G-1)
\end{equation}
For a chi-squared distribution, it is well known that the following holds:
\begin{equation}
\mathbb{E}[\sqrt{X}]=\sqrt{2}\frac{\Gamma(\frac{k+1}{2})}{\Gamma(\frac{k}{2})}.
\end{equation} 
By substituting $k=G-1$, we can obtain that:
\begin{equation}
\mathbb{E}[s_G]=\frac{\sigma}{\sqrt{G-1}} \cdot \mathbb{E}[\sqrt{\frac{(G-1)s_G^2}{\sigma^2}}]=\sqrt{2}\cdot\frac{\sigma}{\sqrt{G-1}}\cdot\frac{\Gamma(\frac{G}{2})}{\Gamma(\frac{G-1}{2})}.
\end{equation}
When $G=8$, we have that:
\begin{equation}
R_{\text{RelBias}}(8)=\frac{\mathbb{E}[s_G]-\sigma}{\sigma}=\sqrt{\frac{2}{7}}\frac{\Gamma(4)}{\Gamma(3.5)}-1\approx 0.965-1=-0.035.
\end{equation}
When $G=16$, we have that:
\begin{equation}
R_{\text{RelBias}}(16)=\frac{\mathbb{E}[s_G]-\sigma}{\sigma}=\sqrt{\frac{2}{15}}\frac{\Gamma(8)}{\Gamma(7.5)}-1\approx 0.983-1=-0.017,
\end{equation}
which completes the proof.
\end{proof}
Therefore, the assumption that ``the group number $G$ is large enough'' is justified in practice. With this established, we now proceed to prove Proposition 1.
\begin{proposition}
Suppose the group number $G$ is large enough, and we set:
\[
\sigma_k = \mathbb{E}_{(\hat{z_0^1}, \hat{z_0^2},\cdots,\hat{z_0^G})\sim p_{\theta_{\text{old}}}}\\\Big[\text{std}\Big(r(c,\hat{z_0^1}; \alpha, p_k)_1, \\ r(c,\hat{z_0^2};\alpha, p_k)_2, \cdots, r(c,\hat{z_0^G}; \alpha, p_k)_G\Big)\Big]
\]
denote the expected reward standard deviation under the $k$-th discrete reward histogram. Then the expected standardized GRPO advantage satisfies:
\begin{equation}
\mathbb{E}_{\alpha\sim \boldsymbol{P}, (\hat{z_0^2},\cdots,\hat{z_0^G})\sim p_{\theta_{\text{old}}}(\cdot|c)}\Big[\hat{A}_{\text{std}}^{p_{\theta_{\text{old}}}}(c, T, \hat{z_0^1})\Big]
= \frac{G-1}{G} \cdot \sum \limits_{j = 1}^{n}\frac{1}{\sigma_j + \epsilon}\cdot p_j \cdot (r(c,z_0;\alpha,p_j)-\eta_j).
\end{equation}
\end{proposition}
\begin{proof}
According to the definition of advantage function, we further have the true advantage estimation as follows:
\begin{equation}
\begin{split}
A^{p_{\theta_{\text{old}}}}(c,t,\hat{z_t},\hat{z}_{t-1})  = & Q^{p_{\theta_{\text{old}}}}(c,t,\hat{z_t},\hat{z}_{t-1}) - V^{p_{\theta_{\text{old}}}}(c,t,\hat{z_t})\\
=&\mathbb{E}_{\alpha\sim \boldsymbol{P}}[r(c,\hat{z_t};\alpha)] - \mathbb{E}_{\alpha\sim \boldsymbol{P}, \hat{z_t}'}[r(c,\hat{z_t}';\alpha)].
\end{split}
\end{equation}
Hence, we can derive that:
\begin{equation}
\begin{split}
A^{p_{\theta_{\text{old}}}}(c, \boldsymbol{\hat{P}}) = & \sum \limits_{t = 1}^{T}A^{p_{\theta_{\text{old}}}}(c,t,\hat{z_t},\hat{z}_{t-1})  \\ = & \mathbb{E}_{t\sim\mathcal{U}(1,T)}\bigg[ \mathbb{E}_{\alpha\sim \boldsymbol{P}}\Big(r(c, \hat{z_t}; \alpha) -\mathbb{E}_{z'_t}r(c,\hat{z_t}'; \alpha)\Big)\bigg].
\end{split}
\end{equation}

Define the advantage estimator that \underline{excludes standard deviation term} for GRPO sample estimation as $\hat{A}_{\text{w/o-std}}^{p_{\theta_{\text{old}}}}(c, T, \hat{z_0})$. Then, consider the advantage for the first sample $\hat{z_0^1}$:
\begin{equation}
\begin{split}
&\mathbb{E}_{\alpha\sim \boldsymbol{P}, (\hat{z_0^2},\cdots,\hat{z_0^G})\sim p_{\theta_{\text{old}}}(\cdot|c)}\Big[\hat{A}_{\text{w/o-std}}^{p_{\theta_{\text{old}}}}(c, T, \hat{z_0^1})\Big]\\
=&\mathbb{E}_{\alpha\sim \boldsymbol{P}, (\hat{z_0^2},\cdots,\hat{z_0^G})\sim p_{\theta_{\text{old}}}(\cdot|c)}\Big[r(c,\hat{z_0^1}; \alpha)_1- \frac{1}{G}\cdot\sum\limits_{m = 1}^{G}r(c,\hat{z_0^m}; \alpha)_m\Big]\\
=&\mathbb{E}_{\alpha\sim \boldsymbol{P}} \Big[\frac{G-1}{G}r(c,\hat{z_0^1}; \alpha)_1 - \frac{G-1}{G}\mathbb{E}_{\hat{z_0^{\text{any}}}}r(c,\hat{z_0^{\text{any}}};\alpha)_{\text{any}}\Big]\\
=&\frac{G-1}{G}\cdot\mathbb{E}_{\alpha\sim \boldsymbol{P}} \Big[r(c,\hat{z_0^1}; \alpha)_1 - \mathbb{E}_{\hat{z_0^{\text{any}}}}r(c,\hat{z_0^{\text{any}}};\alpha)_{\text{any}}\Big]\\
=&\frac{G-1}{G} \cdot A^{p_{\theta_{\text{old}}}}(c, T, \hat{z_0}) = \frac{G-1}{G} A^{p_{\theta_{\text{old}}}}(c, \boldsymbol{\hat{P}}),
\end{split}
\end{equation}
where $\hat{z_0^{\text{any}}}$ refer to any generated output other than $z_0^1$, with $r(c,\hat{z_0^{\text{any}}};\alpha)_{\text{any}}$ indicating the associated reward.
For sufficiently large $G$, the sample-based advantage estimator provides an accurate approximation of the true advantage.

Moreover, let's consider \underline{the expected advantage with standard normalization}. 
As the group number $G$ is large enough, we have can make the following approximation:
\begin{equation}
\mathbb{E}_{(\hat{z_0^2},\cdots,\hat{z_0^G})\sim p_{\theta_{\text{old}}}}\Big[\text{std}\Big(r(c,\hat{z_0^1}; \alpha, p_k)_1, r(c,\hat{z_0^2}; \alpha, p_k)_2, \cdots, r(c,\hat{z_0^G}; \alpha, p_k)_G\Big)\Big]  \approx \sigma_k.
\end{equation}
Therefore, \underline{the advantage with standard derivation} item for the GRPO sample estimation can be expressed as follows:
\begin{equation}
\begin{split}
&\mathbb{E}_{\alpha\sim \boldsymbol{P}, (\hat{z_0^2},\cdots,\hat{z_0^G})\sim p_{\theta_{\text{old}}}(\cdot|c)}\Big[\hat{A}_{\text{std}}^{p_{\theta_{\text{old}}}}(c, T, \hat{z_0^1})\Big]\\
= &\mathbb{E}_{\alpha\sim \boldsymbol{P}, (\hat{z_0^2},\cdots,\hat{z_0^G})\sim p_{\theta_{\text{old}}}(\cdot|c)} \Big[\frac{r(c,\hat{z_0^1}; \alpha)_1-\text{mean}(r(c,\hat{z_0^1}; \alpha)_1,\cdots,r(c,\hat{z_0^G}; \alpha)_G)}{\text{std}(r(c,\hat{z_0^1}; \alpha)_1,\cdots,r(c,\hat{z_0^G}; \alpha)_G)+\epsilon}\Big]\\
= &\mathbb{E}_{\alpha\sim \boldsymbol{P}} \frac{G-1}{G}\cdot \frac{ A^{p_{\theta_{\text{old}}}}(c,\boldsymbol{\hat{P}})}{\boldsymbol{\sigma} + \epsilon}\\
= &\frac{G-1}{G} \cdot \sum \limits_{j = 1}^{n}\frac{1}{\sigma_j + \epsilon}\cdot p_j \cdot (r(c,z_0;\alpha,p_j)-\eta_j),
\end{split}
\end{equation}
which completes the proof.
\end{proof}

\section{Proof of Theorem 2}
\label{Proof of Theorem 2}
Before proving \Cref{optimal ascent}, we introduce the following lemma on equivalence class induced by the policy gradient operator.
\begin{lemma}
\label{equivalent_function}
Let $\pi_{\theta}(y|x)$ be a differentiable conditional probability distribution such that:
\begin{equation}
\forall x, \; \int_{\mathcal{Y}}\pi_{\theta}(y|x) dy = 1.
\end{equation}
For any integrable function $f(x, y): \mathcal{X}\times\mathcal{Y}\rightarrow\mathbb{R}$, the gradient operator can be expressed as:
\begin{equation}
\mathcal{G}(f(x,y)):=\mathbb{E}_{\pi_{\theta}}[\nabla_{\theta}\log\pi_{\theta}(y|x)\cdot f(x, y)].
\end{equation}
Then, for any measurable function $b(x): \mathcal{X}\rightarrow\mathbb{R}$, we have that:
\begin{equation}
\mathcal{G}(f(x, y)) = \mathcal{G}(f(x, y) + b(x)).
\end{equation}
Equivalently, under the gradient operator $\mathcal{G}(\cdot)$, we have that: $f(x, y) \sim f(x, y) + b(x)$.
\end{lemma}
\begin{proof}
According to the definition, we have that:
\begin{equation}
\begin{split}
\mathcal{G}(f(x,y))=&\int_{\mathcal{Y}}f(x, y)\cdot\delta\pi(y|x) dy = \int_{\mathcal{Y}}f(x, y)\cdot\pi(y|x)\cdot\delta\log\pi(y|x) dy \\=& \mathbb{E}_{\pi_{\theta}(y|x)}[f(x, y)\nabla_{\theta}\log\pi_{\theta}(y|x)].
\end{split}
\end{equation}
Furthermore, for the function $b(x)$, which does not depend on $y$, we have that:
\begin{equation}
\begin{split}
\mathcal{G}(b(x)) =& \mathbb{E}_{\pi_{\theta}(y|x)}[b(x)\nabla_{\theta}\log\pi_{\theta}(y|x)] = \int_{\mathcal{Y}}b(x)\cdot\pi(y|x)\cdot\delta\log\pi(y|x) dy\\ =& \int_{\mathcal{Y}}b(x)\cdot\delta\pi(y|x) dy
= b(x)\cdot\int_{\mathcal{Y}}\delta\pi(y|x) dy =0
\end{split}
\end{equation}
Hence, we readily obtain that $\mathcal{G}(f(x, y)) = \mathcal{G}(f(x, y) + b(x))$, which completes the proof.
\end{proof}
\begin{theorem}
Considering all perturbations $\delta\pi(y|x)$ that satisfy the constraint in \Cref{constrain}, and noting that the functional $\boldsymbol{\mathcal{F}}(\pi(y|x))$ is defined on the measure space, we assume that the policy $\pi(y|x)$ is absolutely continuous. Then, the optimal ascent is characterized by the following result:
\begin{equation}
\delta\pi(y|x)\propto\pi_0(y|x) \cdot A(x, y),
\end{equation}
where $\pi_0(y|x)$ denotes the current policy function, $A(x, y) = r(x,y) + \lambda(x)$ represents the advantage function class.
\end{theorem}
\begin{proof}
Firstly, we reformulate \Cref{constrained functional optimization problem} as follows:
\begin{equation}
\label{optimal problem}
\max\limits_{\pi(y|x)}\{ \boldsymbol{\mathcal{F}}(\pi(y|x))-\frac{1}{\eta}\cdot\mathbb{D}_{\text{KL}}(\pi(y|x)||\pi_0(y|x)) \}.
\end{equation}
Under these settings, we consider a single local update $\pi(y|x) = \pi_0(y|x)+\delta\pi(y|x)$, where the value of $||\delta\pi(y|x)||$ is very small.

Firstly, we perform the calculation for $\mathbb{D}_{\text{KL}}(\pi(y|x)||\pi_0(y|x))$. Given a fixed input $x$, KL divergence of the conditional policy is defined as:
\begin{equation}
\mathbb{D}_{\text{KL}}(\pi(y|x)||\pi_0(y|x))=\int_{\mathcal{Y}}\pi(y|x)\log\frac{\pi(y|x)}{\pi_0(y|x)}dy.
\end{equation}
Furthermore, we perform the Taylor expansion for the logarithmic term:
\begin{equation}
\begin{split}
\log\frac{\pi(y|x)}{\pi_0(y|x)} &= \log\frac{\pi_0(y|x)+\delta\pi(y|x)}{\pi_0(y|x)} = \log\bigg(1+\frac{\delta\pi(y|x)}{\pi_0(y|x)}\bigg) \\&\xlongequal{||\delta\pi||\rightarrow 0} \frac{\delta\pi(y|x)}{\pi_0(y|x)}-\frac{1}{2}\bigg(\frac{\delta\pi(y|x)}{\pi_0(y|x)}\bigg)^2 + o\big((\delta\pi(y|x))^2\big).
\end{split}
\end{equation}
Therefore, we can obtain that:
\begin{equation}
\begin{split}
\mathbb{D}_{\text{KL}}(\pi(y|x)||\pi_0(y|x)) &= \int_{\mathcal{Y}} (\pi_0(y|x)+\delta\pi(y|x))\Big[\frac{\delta\pi(y|x)}{\pi_0(y|x)}-\frac{1}{2}\Big(\frac{\delta\pi(y|x)}{\pi_0(y|x)}\Big)^2\Big] dy + o\big((\delta\pi(y|x))^2\big)\\ &=\int_{\mathcal{Y}} \Big(\delta\pi(y|x) + \frac{(\delta\pi(y|x))^2 }{\pi_0(y|x)} - \frac{1}{2} \frac{(\delta\pi(y|x))^2 }{\pi_0(y|x)}\Big) dy + o\big((\delta\pi(y|x))^2\big)\\ &=\int_{\mathcal{Y}} \Big(\delta\pi(y|x) + \frac{1}{2} \frac{(\delta\pi(y|x))^2 }{\pi_0(y|x)}\Big) dy + o\big((\delta\pi(y|x))^2\big) \\&= \int_{\mathcal{Y}}\frac{1}{2} \frac{(\delta\pi(y|x))^2 }{\pi_0(y|x)} dy + o\big((\delta\pi(y|x))^2\big).
\end{split}
\end{equation}
We perform a first-order expansion of the return functional in \Cref{optimal problem} as follows:
\begin{equation}
\begin{split}
\boldsymbol{\mathcal{F}}(\pi(y|x)) &= \boldsymbol{\mathcal{F}}(\pi_0(y|x) + \delta\pi(y|x)) \\&= \boldsymbol{\mathcal{F}}(\pi_0(y|x)) + \delta \boldsymbol{\mathcal{F}}(\pi_0(y|x); \delta\pi(y|x)) + o(\delta\boldsymbol{\mathcal{F}}(\pi_0(y|x); \delta\pi(y|x))).
\end{split}
\end{equation}
Hence, the functional local update problem can be further formulated as:
\begin{equation}
\max\limits_{\delta\pi(y|x)}\{\delta \boldsymbol{\mathcal{F}}(\pi_0(y|x); \delta\pi(y|x)) - \frac{1}{2\eta}\int_{\mathcal{Y}}\frac{(\delta\pi(y|x))^2 }{\pi_0(y|x)} dy\}
\end{equation}
To further solve this optimization problem, we incorporate \Cref{constrain} as a constraint. This leads to the construction of the following Lagrangian functional:
\begin{equation}
\label{Lagrangian functional}
\boldsymbol{\mathcal{L}}(\delta\pi(y|x), \lambda(x)) = \int_{\mathcal{Y}}r(x, y)\cdot\delta\pi(y|x) dy - \frac{1}{2\eta}\int_{\mathcal{Y}}\frac{(\delta\pi(y|x))^2 }{\pi_0(y|x)} dy + \int_{\mathcal{X}}\lambda(x)\Big(\int_{\mathcal{Y}}\delta\pi(y|x)dy\Big)dx
\end{equation}
Moreover, let's perform the Fréchet derivative of $\delta\pi(y|x)$ in \Cref{Lagrangian functional}. To summarize, we perform the computation in the following manner. Introduce any admissible perturbation function $h(x,y)$, we have that:
\begin{equation}
\int_{\mathcal{Y}}\frac{\delta\boldsymbol{\mathcal{L}}(\delta\pi(y|x),\lambda(x))}{\delta(\delta\pi(y|x))}h(x, y)dy = \frac{d}{d\epsilon}\boldsymbol{\mathcal{L}}(\delta\pi(y|x)+\epsilon h(x, y),\lambda(x))\bigg|_{\epsilon = 0}.
\end{equation}
We proceed term by term to compute the functional derivatives. For the first term of \Cref{Lagrangian functional}, we have that:
\begin{equation}
\boldsymbol{\mathcal{L}_1}(\delta\pi(y|x)+\epsilon h(x, y), \lambda(x)) = \int_{\mathcal{Y}} r(x, y)(\delta\pi(y|x)+\epsilon h(x, y)) dy.
\end{equation}
Hence, we can easily get that:
\begin{equation}
\label{L_1_derivative}
\frac{\delta\boldsymbol{\mathcal{L}_1}(\delta\pi(y|x),\lambda(x))}{\delta(\delta\pi(y|x))} = r(x, y).
\end{equation}
Furthermore, for the second term, we have that:
\begin{equation}
\boldsymbol{\mathcal{L}_2}(\delta\pi(y|x)+\epsilon h(x, y), \lambda(x)) = - \frac{1}{2\eta}\int_{\mathcal{Y}}\frac{(\delta\pi(y|x)+\epsilon h(x, y))^2}{\pi_0(y|x)} dy.
\end{equation}
In particular, we have that $(\delta\pi(y|x)+\epsilon h(x, y))^2 = (\delta\pi(y|x))^2 + 2\epsilon\cdot\delta\pi(y|x)\cdot h(x, y) + o(\epsilon)$, and we can get:
\begin{equation}
\frac{d}{d\epsilon}\boldsymbol{\mathcal{L}_2}(\delta\pi(y|x)+\epsilon h(x, y),\lambda(x))\bigg|_{\epsilon = 0} = -\frac{1}{\eta}\cdot\int_{\mathcal{Y}}\frac{\delta\pi(y|x)}{\pi_0(y|x)}\cdot h(x,y).
\end{equation}
Therefore, we can get that:
\begin{equation}
\label{L_2_derivative}
\frac{\delta\boldsymbol{\mathcal{L}_2}(\delta\pi(y|x),\lambda(x))}{\delta(\delta\pi(y|x))} = -\frac{1}{\eta}\cdot\frac{\delta\pi(y|x)}{\pi_0(y|x)}.
\end{equation}
Then, for the third term, which also serves as the constraint term:
\begin{equation}
\begin{split}
\boldsymbol{\mathcal{L}_3}(\delta\pi(y|x)+\epsilon h(x, y), \lambda(x)) =& \int_{\mathcal{X}}\lambda(x)\Big(\int_{\mathcal{Y}}\delta\pi(y|x) +\epsilon h(x, y) dy\Big)dx
\\=&\int_{\mathcal{X}}\lambda(x)\int_{\mathcal{Y}}\delta\pi(y|x) dy dx + \epsilon \cdot \int_{\mathcal{X}}\lambda(x)\int_{\mathcal{Y}}h(x, y)dy dx.
\end{split}
\end{equation}
Then, it can be readily derived that:
\begin{equation}
\frac{d}{d\epsilon}\boldsymbol{\mathcal{L}_3}(\delta\pi(y|x)+\epsilon h(x, y),\lambda(x))\bigg|_{\epsilon = 0} = \int_{\mathcal{X}}\lambda(x)\int_{\mathcal{Y}}h(x, y)dy dx.
\end{equation}
Thus, we can obtain that:
\begin{equation}
\label{L_3_derivative}
\frac{\delta\boldsymbol{\mathcal{L}_3}(\delta\pi(y|x),\lambda(x))}{\delta(\delta\pi(y|x))} = \lambda(x).
\end{equation}
Combining the results in \Cref{L_1_derivative}, \Cref{L_2_derivative} and \Cref{L_3_derivative}, the Fréchet derivative of $\boldsymbol{\mathcal{L}}(\delta\pi(y|x), \lambda(x))$ can be expressed as:
\begin{equation}
\frac{\delta\boldsymbol{\mathcal{L}}(\delta\pi(y|x),\lambda(x))}{\delta(\delta\pi(y|x))} = r(x, y) -\frac{1}{\eta}\cdot\frac{\delta\pi(y|x)}{\pi_0(y|x)}  + \lambda(x).
\end{equation}
By setting the Fréchet derivative of $\boldsymbol{\mathcal{L}}(\delta\pi(y|x), \lambda(x))$ to zero and rearranging the terms, we obtain the following expression:
\begin{equation}
\delta\pi(y|x) = \eta\cdot\pi_0(y|x) \cdot (r(x, y)+\lambda(x)).
\end{equation}
According to \Cref{equivalent_function}, we can easily draw the conclusion that $r(x, y)+\lambda(x) \sim r(x, y)$. Moreover, we refer to functions in this equivalence function class as $A(x, y)$, which corresponds to \underline{advantage function class}, while $\lambda(x)$ can be interpreted as the baseline.
Hence, we can obtain that:
\begin{equation}
\delta\pi(y|x)\propto\pi_0(y|x) \cdot A(x, y),
\end{equation}
which completes the proof.
\end{proof}
\section{Proof of Theorem 3}
\label{Proof of Theorem 3}
\begin{theorem}
Under the same conditions as in \Cref{optimal ascent}, after transforming to the tangent space of the probability simplex as expressed in \Cref{gamma-weighted variational metric}, the optimal ascent is characterized by:
\begin{equation}
\delta\pi(y|x)\propto\pi_0(y|x) \cdot \text{sign}(A(x, y))|A(x,y)|^{1+\gamma},
\end{equation}
where $\pi_0(y|x)$ denotes the current policy function, $\text{sign}(\cdot)$ denotes the sign function, $A(x, y)$ represents the advantage function class.
\end{theorem}
\begin{proof}
In this context, we examine the functional local update problem equipped with the $\gamma$-weighted variational metric:
\begin{equation}
\label{monotone non-decreasing mapping_'}
\max\limits_{\delta\pi(y|x)}\{\delta \boldsymbol{\mathcal{F}}(\pi_0(y|x); \delta\pi(y|x)) - \frac{1}{2\eta}\int_{\mathcal{Y}}\frac{(\delta\pi(y|x))^2 }{|A(x,y)|^\gamma\cdot\pi_0(y|x)} dy\}.
\end{equation}
By imposing \Cref{constrain} as a constraint, we construct the following Lagrangian functional:
\begin{equation}
\begin{split}
\label{Lagrangian functional weighted}
&\boldsymbol{\mathcal{L'}}(\delta\pi(y|x), \lambda'(x)) \\= &\int_{\mathcal{Y}}r(x, y)\cdot\delta\pi(y|x) dy - \frac{1}{2\eta}\int_{\mathcal{Y}}\frac{(\delta\pi(y|x))^2 }{|A(x,y)|^\gamma\cdot\pi_0(y|x)} dy + \int_{\mathcal{X}}\lambda'(x)\Big(\int_{\mathcal{Y}}\delta\pi(y|x)dy\Big)dx
\end{split}
\end{equation}
Then, performing the Fréchet derivative of $\delta\pi(y|x)$ in \Cref{Lagrangian functional weighted}. Similarly, for any admissible perturbation function $h(x,y)$, we have that:
\begin{equation}
\int_{\mathcal{Y}}\frac{\delta\boldsymbol{\mathcal{L'}}(\delta\pi(y|x),\lambda'(x))}{\delta(\delta\pi(y|x))}h(x, y)dy = \frac{d}{d\epsilon}\boldsymbol{\mathcal{L}'}(\delta\pi(y|x)+\epsilon h(x, y),\lambda'(x))\bigg|_{\epsilon = 0}.
\end{equation}
We proceed term by term to compute the functional derivatives. For the first term of \Cref{Lagrangian functional weighted}, we have that:
\begin{equation}
\boldsymbol{\mathcal{L}'_1}(\delta\pi(y|x)+\epsilon h(x, y), \lambda'(x)) = \int_{\mathcal{Y}} r(x, y)(\delta\pi(y|x)+\epsilon h(x, y)) dy.
\end{equation}
Hence, we can get that:
\begin{equation}
\label{L_1_derivative_'}
\frac{\delta\boldsymbol{\mathcal{L}'_1}(\delta\pi(y|x),\lambda'(x))}{\delta(\delta\pi(y|x))} = r(x, y).
\end{equation}
Then, for the second term, we have that:
\begin{equation}
\boldsymbol{\mathcal{L}'_2}(\delta\pi(y|x)+\epsilon h(x, y), \lambda'(x)) = - \frac{1}{2\eta}\int_{\mathcal{Y}}\frac{(\delta\pi(y|x)+\epsilon h(x, y))^2}{|A(x,y)|^\gamma\cdot\pi_0(y|x)} dy.
\end{equation}
In particular, we have that $(\delta\pi(y|x)+\epsilon h(x, y))^2 = (\delta\pi(y|x))^2 + 2\epsilon\cdot\delta\pi(y|x)\cdot h(x, y) + o(\epsilon)$, and we can get that:
\begin{equation}
\frac{d}{d\epsilon}\boldsymbol{\mathcal{L}'_2}(\delta\pi(y|x)+\epsilon h(x, y),\lambda'(x))\bigg|_{\epsilon = 0} = -\frac{1}{\eta}\cdot\int_{\mathcal{Y}}\frac{\delta\pi(y|x)}{|A(x,y)|^\gamma\cdot\pi_0(y|x)}\cdot h(x,y).
\end{equation}
Therefore, we can get that:
\begin{equation}
\label{L_2_derivative_'}
\frac{\delta\boldsymbol{\mathcal{L}'_2}(\delta\pi(y|x),\lambda'(x))}{\delta(\delta\pi(y|x))} = -\frac{1}{\eta}\cdot\frac{\delta\pi(y|x)}{|A(x,y)|^\gamma\cdot\pi_0(y|x)}.
\end{equation}
Then, for the third term, which also serves as the constraint term:
\begin{equation}
\begin{split}
\boldsymbol{\mathcal{L}'_3}(\delta\pi(y|x)+\epsilon h(x, y), \lambda'(x)) =& \int_{\mathcal{X}}\lambda'(x)\Big(\int_{\mathcal{Y}}\delta\pi(y|x) +\epsilon h(x, y) dy\Big)dx
\\=&\int_{\mathcal{X}}\lambda'(x)\int_{\mathcal{Y}}\delta\pi(y|x) dy dx + \epsilon \cdot \int_{\mathcal{X}}\lambda'(x)\int_{\mathcal{Y}}h(x, y)dy dx.
\end{split}
\end{equation}
Then, it can be readily derived that:
\begin{equation}
\frac{d}{d\epsilon}\boldsymbol{\mathcal{L}'_3}(\delta\pi(y|x)+\epsilon h(x, y),\lambda'(x))\bigg|_{\epsilon = 0} = \int_{\mathcal{X}}\lambda'(x)\int_{\mathcal{Y}}h(x, y)dy dx.
\end{equation}
Thus, we can obtain that:
\begin{equation}
\label{L_3_derivative_'}
\frac{\delta\boldsymbol{\mathcal{L}'_3}(\delta\pi(y|x),\lambda'(x))}{\delta(\delta\pi(y|x))} = \lambda'(x).
\end{equation}
Combining the results in \Cref{L_1_derivative_'}, \Cref{L_2_derivative_'} and \Cref{L_3_derivative_'}, the Fréchet derivative of $\boldsymbol{\mathcal{L}'}(\delta\pi(y|x), \lambda'(x))$ can be expressed as:
\begin{equation}
\frac{\delta\boldsymbol{\mathcal{L}'}(\delta\pi(y|x),\lambda'(x))}{\delta(\delta\pi(y|x))} = r(x, y) -\frac{1}{\eta}\cdot\frac{\delta\pi(y|x)}{|A(x,y)|^\gamma\cdot\pi_0(y|x)}  + \lambda'(x).
\end{equation}
Setting the Fréchet derivative of $\boldsymbol{\mathcal{L}'}(\delta\pi(y|x), \lambda'(x))$ to zero and rearranging the terms, we obtain the following expression:
\begin{equation}
\delta\pi(y|x) = \eta\cdot|A(x,y)|^\gamma\cdot \pi_0(y|x)\cdot (r(x, y)+\lambda'(x)).
\end{equation}
Similarly, according to \Cref{equivalent_function}, we can readily obtain that $r(x, y)+\lambda'(x) \sim r(x, y) \sim A(x, y)$.
Hence, we have:
\begin{equation}
\delta\pi(y|x)\propto\pi_0(y|x) \cdot |A(x,y)|^\gamma \cdot A(x, y).
\end{equation}
Let $\text{sign}(\cdot)$ denote the indicator function, and it is straightforward to obtain that:
\begin{equation}
|A(x,y)|^\gamma \cdot A(x, y) = \text{sign}(A(x, y)) \cdot |A(x,y)|^{1+\gamma},
\end{equation}
which completes the proof.
\end{proof}
\section{Trust-Region Bounds on \texorpdfstring{$\gamma$}{gamma}}
\label{Trust-Region Bounds on gamma}
Moreover, for training stability, $\gamma$ should be properly bounded and not excessively large. This is guaranteed through a KL-divergence constraint. Accordingly, we consider the following restriction:
\begin{equation}
\label{restriction}
\mathbb{E}_{x\sim\upsilon^{\pi}}{\mathbb{D}_{\text{KL}}(\pi(y|x)||\pi_{0}(y|x))} \leq \zeta.
\end{equation}
The super-linearly reshaped advantage in \Cref{reshaped super-linear advantage} modifies the gradient at each update step, which in turn induces a change in the KL divergence. Consequently, we investigate how large $\gamma$ can be under a given KL-divergence threshold in \Cref{restriction}, while still guaranteeing that each update step remains within the trust region. \Cref{implicit range} characterizes the implicit range of $\gamma$ determined by an arbitrary reward function, under the following two assumptions.
\begin{assumption}[Differentiable Policy] 
\label{Differentiable Policy}
The policy $\pi_{\theta}(y|x)$ is differentiable in $\theta$, and:
\begin{equation}
\mathbb{E}_{\pi_{\theta}}[||\nabla_{\theta}\log\pi_{\theta}(y|x)||^2] \leq G_{\text{max}}.
\end{equation}
\end{assumption}
\begin{assumption}[Finite Reward Moments]
\label{Finite Reward Moments}
For a given $\gamma\geq 0$, the $2(1+\gamma)$-th moment of the reward exists:
\begin{equation}
\mathbb{E}[|\Delta r|^{2(1+\gamma)}]<\infty.
\end{equation}
\end{assumption}
Notably, Assumption \ref{Finite Reward Moments} holds for bounded or sub-Gaussian rewards, which is often the case for rewards in text-to-image tasks. Based on these, we state the following theorem.
\begin{theorem}
\label{implicit range}
Let's consider the first-order update $\Delta\theta=\tau\cdot\widetilde{g}$, where $\widetilde{g}$ represents the gradient, and we have:
\begin{equation}
\widetilde{g} = \mathbb{E}_{\pi_{\theta}}[|\Delta r|^\gamma\cdot\Delta r\cdot\nabla_{\theta}\log\pi_{\theta}]. 
\end{equation}
We define Hessian matrix of average KL divergence between the policy before and after the update:
\begin{equation}
\mathcal{H} = \mathcal{H}[\mathbb{E}_{x\sim\upsilon^{\pi}}\mathbb{D}_{\text{KL}}(\pi(y|x)||\pi_0(y|x))],
\end{equation}
with the maximum eigenvalue of the matrix being $\lambda_{\text{max}}(\mathcal{H})$.
Then, the $2(1+\gamma)$-th moment of the reward have the following upper bound:
\begin{equation}
M_{2(1+\gamma)} = \mathbb{E}[|\Delta r|^{2(1+\gamma)}]\leq \frac{2\zeta}{\tau^2\cdot\lambda_{\text{max}}(\mathcal{H})\cdot G_{\text{max}}}.
\end{equation}
\end{theorem}
\begin{proof}
For the first-order update, e.g. $\Delta\theta=\tau\cdot\widetilde{g}$, we can easily obtain the following estimate for the KL divergence:
\begin{equation}
\label{estimate}
\mathbb{E}_{x\sim\upsilon^{\pi}}\mathbb{D}_{\text{KL}}(\pi(y|x)||\pi_0(y|x)) \approx \frac{\tau^2}{2}\cdot \widetilde{g}^T\mathcal{H}\widetilde{g}.
\end{equation}
By definition, $\lambda_{\text{max}}(\mathcal{H})$ corresponds to the largest eigenvalue of the Hessian matrix $\mathcal{H}$; hence, one can easily show that:
\begin{equation}
\label{Hessian matrix inequality}
\widetilde{g}^T\mathcal{H}\widetilde{g}\leq\lambda_{\text{max}}(\mathcal{H})\cdot||\widetilde{g}||^2.
\end{equation}
Let's consider upper bound of the gradient’s $l_2$-norm:
\begin{equation}
||\widetilde{g}||^2 = ||\mathbb{E}_{\pi_{\theta}}[|\Delta r|^\gamma\cdot\Delta r\cdot\nabla_{\theta}\log\pi_{\theta}]||^2
\end{equation}
Applying the Jensen inequality, it is readily to obtain that:
\begin{equation}
\label{Jensen inequality}
\begin{split}
||\mathbb{E}_{\pi_{\theta}}[|\Delta r|^\gamma\cdot\Delta r\cdot\nabla_{\theta}\log\pi_{\theta}]||^2 &\leq \mathbb{E}_{\pi_{\theta}} || [|\Delta r|^\gamma\cdot\Delta r\cdot\nabla_{\theta}\log\pi_{\theta}]||^2\\
&= \underbrace{[\mathbb{E} |\Delta r|^{2(1+\gamma)}]}_{M_{2(1+\gamma)}} \cdot \underbrace{[\mathbb{E_{\pi_{\theta}}}||\nabla_{\theta}\log\pi_{\theta}||^2]}_{G_{\text{max}}}.
\end{split}
\end{equation}
Putting together \Cref{estimate}, \Cref{Hessian matrix inequality} and \Cref{Jensen inequality}, we arrive at:
\begin{equation}
\mathbb{E}_{x\sim\upsilon^{\pi}}\mathbb{D}_{\text{KL}}(\pi(y|x)||\pi_0(y|x))\leq\frac{\tau^2}{2}\lambda_{\text{max}}(\mathcal{H})\cdot M_{2(1+\gamma)}\cdot G_{\text{max}}.
\end{equation}
Accordingly, to maintain a safe update, we require that:
\begin{equation}
\frac{\tau^2}{2}\lambda_{\text{max}}(\mathcal{H})\cdot M_{2(1+\gamma)}\cdot G_{\text{max}}\leq \zeta.
\end{equation}
Rearranging the inequality yields an implicit upper bound on $\gamma$, determined by the $2(1+\gamma)$-th moment of the reward:
\begin{equation}
M_{2(1+\gamma)} \leq \frac{2\zeta}{\tau^2\cdot\lambda_{\text{max}}(\mathcal{H})\cdot G_{\text{max}}},
\end{equation}
which completes the proof.
\end{proof}
Furthermore, we consider imposing commonly utilized assumptions on the reward model. Firstly, most reward function models exhibit boundedness, with CLIP Score \citep{radford2021learning}, HPS-v1 \citep{wu2023human}, HPS-v2 \citep{wu2023human}, PickScore \citep{kirstain2023pick} and Aesthetic Score \citep{schuhmann2022laion} serving as representations. Accordingly, \Cref{Bounded Rewards} presents the corresponding result under the bounded-reward setting. Furthermore, as most reward function models are trained on millions of samples, their outputs can be reasonably approximated as sub-Gaussian distribution, with the resulting analysis summarized in \Cref{Sub-Gaussian Rewards}.
\begin{corollary}[Bounded Rewards]
\label{Bounded Rewards}
Suppose we have that $|\Delta r|\leq R_{\text{max}}$, then it can be derived that:
\begin{equation}
\gamma\leq\frac{\log(\frac{2\zeta}{\tau^2\cdot\lambda_{\text{max}}(\mathcal{H})\cdot G_{\text{max}}})}{2\log R_{\text{max}}} -  1,
\end{equation}
showing that $\gamma$ is strictly bounded as a function of the reward upper bound.
\end{corollary}
\begin{proof}
As we have that $|\Delta r|\leq R_{\text{max}}$, it is readily to obtain that:
\begin{equation}
\mathbb{E}[|\Delta r|^{2(1+\gamma)}] \leq R_{\text{max}}^{2(1+\gamma)}.
\end{equation}
Hence, we set that:
\begin{equation}
R_{\text{max}}^{2(1+\gamma)} \leq \frac{2\zeta}{\tau^2\cdot\lambda_{\text{max}}(\mathcal{H})\cdot G_{max}}.
\end{equation}
By refining this inequality, we readily obtain that:
\[
\gamma\leq\frac{\log(\frac{2\zeta}{\tau^2\cdot\lambda_{\text{max}}(\mathcal{H})\cdot G_{\text{max}}})}{2\log R_{\text{max}}} -  1,
\]
which completes the proof.
\end{proof}
\begin{corollary}[Sub-Gaussian Rewards]
\label{Sub-Gaussian Rewards}
Suppose that $\Delta r$ follows a sub-Gaussian distribution and we have the $\psi_2$-Orlicz norm defined as:
\begin{equation}
||\Delta r||_{\psi_2}:=\inf\{C>0: \mathbb{E}[\exp(\frac{(\Delta r)^2}{C^2})\leq 2]\}\leq K_r.
\end{equation}
Then, the implicit range of $\gamma$ is given by:
\begin{equation}
(1+\gamma)[\log(1+\gamma)+\log(2C^2K_r^2)]\leq\log\frac{2\zeta}{\tau^2\lambda_{\text{max}}(\mathcal{H})G_{\text{max}}},
\end{equation}
showing that $\gamma$ can increase safely only at logarithmic scale.
\end{corollary}
\begin{proof}
As shown in the classical literature on high-dimensional probability \citep{vershynin2025high}, it is well known that a random variable $x$ being sub-Gaussian is equivalent to having a finite $\psi_2$-Orlicz norm. And we also have the following inequality:
\begin{equation}
(\mathbb{E}|\Delta r|^k)^\frac{1}{k}\leq C||\Delta r||_{\psi_2}\cdot \sqrt{k}.
\end{equation}
Equivalently, this admits the following formulation, yielding an upper bound on expected higher-order moments of $\Delta r$:
\begin{equation}
\mathbb{E}[|\Delta r|^k]\leq (C\cdot ||\Delta r||_{\psi _2})^k\cdot k^{\frac{k}{2}}.
\end{equation}
Consider the $2(1+\gamma)$-th moment of the reward, and set that $k = 2(1+\gamma)$:
\begin{equation}
\mathbb{E}[|\Delta r|^{2(1+\gamma)}]\leq (C\cdot K_{r})^{2(1+\gamma)}\cdot (2(1+\gamma))^{(1+\gamma)}.
\end{equation}
Hence, we set that:
\begin{equation}
(C\cdot K_{r})^{2(1+\gamma)}\cdot (2(1+\gamma))^{(1+\gamma)} \leq \frac{2\zeta}{\tau^2\lambda_{\text{max}}(\mathcal{H})G_{\text{max}}}.
\end{equation}
Taking logarithms on both sides and applying the necessary simplifications, we can obtain:
\[
(1+\gamma)[\log(1+\gamma)+\log(2C^2K_r^2)]\leq\log\frac{2\zeta}{\tau^2\lambda_{\text{max}}(\mathcal{H})G_{\text{max}}},
\]
which completes the proof.
\end{proof}
In large-scale practical implementations, $\gamma$ can be treated as a trainable parameter that is coupled with the KL divergence, allowing the training process to adaptively adjust it based on the current context. As shown in \Cref{Bounded Rewards} and \Cref{Sub-Gaussian Rewards}, $\gamma$ is upper-bounded and grows at most logarithmically.
\paragraph{Discussion.}  The theoretical characterization validates the soundness of the super-linear advantage shaping mechanism, confirming that it does not compromise the training stability that is critical for fine-tuning large text-to-image models. Moreover, for quantities such as $G_{\text{max}}$ and $\lambda_{\text{max}}(\mathcal{H})$ that are difficult to compute exactly in large-scale settings, we can still obtain efficient and conservative upper bounds with negligible overhead. For $G_{\text{max}}$, it corresponds to an upper bound on the expected squared norm of the policy gradient. Monitoring gradient norms is already a standard practice in modern large-model training pipelines, where it is widely used to detect gradient explosion or vanishing gradients. As a result, we can directly reuse these existing training statistics to derive an empirical, conservative upper bound for $G_{\text{max}}$ at zero additional computational cost. $\mathcal{H}$ is the Hessian of the KL divergence, which is equivalent to the Fisher Information Matrix (FIM) of the policy. To estimate its maximum eigenvalue, it is unnecessary to explicitly instantiate the full FIM. Instead, we can leverage the power iteration method, a well-established matrix-free algorithm \citep{martens2010deep} that requires only matrix-vector multiplications. In practice, this reduces to computing Fisher-vector products (FVPs), whose cost is comparable to a single backward pass. For large models, the method typically converges within 10–20 iterations to yield an accurate and conservative estimate of the dominant eigenvalue. The total overhead is therefore equivalent to only 10–20 backward passes, which is negligible compared with the overall cost of fine-tuning billion-parameter models. Moreover, this estimation need not be performed at every training step; periodic evaluation suffices, further reducing the already modest overhead. Such a paradigm has been widely adopted and validated in prior large-scale optimization work \citep{gupta2018shampoo}, demonstrating its practical feasibility.

\section{Additional Related Works Statement}
\label{Additional Related Work Statement}
In recent years, rapid advances in deep learning have significantly accelerated the development of text-to-image generation \citep{zhang2023text, wu2026visual, mansimov2015generating, reed2016generative, chen2018attention, shoshan2021gan, ding2021cogview, ding2022cogview2, gu2022vector, rombach2022high, podellsdxl, esser2024scaling, flux2024, xie2024show, xie2025show, wu2025qwen,  song2026cologen, song2025query}, establishing it as one of the central applications in the area of computer vision \citep{sun2025reinforcement, shen2024explanatory, yin2025floorplan,  fang2026threading, fang2025viss, wang2026refalign, zhang2026sama}. Despite remarkable capabilities demonstrated by these models, improving human value alignment and instruction-following ability has always been an open challenge. 
\subsection{Reinforcement Post-Training for Large Language Models.} In recent years, rapid progress in machine learning, especially in reinforcement learning \citep{xia2025delay, suncalibration}, has fueled the growing adoption of reinforcement-learning-based post-training paradigms. They are typically first applied to the post-training of large language models (LLMs). As a representative framework, \underline{reinforcement learning from human feedback (RLHF)} \citep{christiano2017deep, ziegler2019fine} provides a principled way to incorporate human preferences into the optimization process, thereby guiding model behavior toward alignment with human values and intended objectives. At a high level, RLHF typically comprises three essential stages \citep{casper2023open}: the collection of human feedback, the construction of a reward model, and the subsequent optimization of the policy. This framework enables human intent to be conveyed implicitly, avoiding the need for explicit reward specification, while exploiting preference-based judgments that are generally more accessible than expert demonstrations. Prior to the adoption of RLHF, they are commonly aligned with human preferences by means of supervised fine-tuning (SFT) on curated demonstration data. In RLHF-based post-training, human feedback is incorporated as reward signals to guide policy optimization, encouraging large language models to produce responses that better align with accuracy, relevance, and informativeness. This iterative learning process enables continual performance improvement and has underpinned the success of many state-of-the-art models, exemplified by OpenAI’s GPT-4 \citep{achiam2023gpt}, Meta’s Llama 3 \citep{touvron2023llama}, Google’s Gemma \citep{team2025gemma}, and so on. Despite their success, RLHF-based approaches often entail high computational costs, demand extensive hyperparameter tuning. In response, methods such as \underline{direct preference optimization (DPO)} \citep{rafailov2023direct} have emerged, which optimize policy models directly with a preference-based loss, achieving results on par with RLHF while reducing training complexity. Thanks to its theoretical elegance and computational efficiency, DPO-style methods have garnered significant attention, inspiring a growing body of research investigating its diverse implementations and applications \citep{hong2024orpo, meng2024simpo, ethayarajh2024kto, pang2024iterative, amini2024direct, xucontrastive, liu2025lipo}. Very recently, the \underline{Group Relative Policy Optimization (GRPO)} \citep{shao2024deepseekmath} paradigm has recently emerged as a focal point in LLM research, owing to its ability to improve explainability in challenging tasks, including mathematical problem solving, code generation, and logical reasoning, as well as to facilitate steady gains in overall model accuracy. Inspired by this paradigm, DeepSeek-R1 \citep{guo2025deepseek} employs a combination of supervised fine-tuning (SFT) and reinforcement learning from verifiable rewards (RLVR), enabling the model to acquire strong reasoning capabilities while relying on a relatively small amount of supervised data. As a significant breakthrough, this paradigm has driven extensive research efforts aimed at improving the efficiency and effectiveness of algorithmic training \citep{yu2025dapo, lin2025cppo, liu2025understanding, yue2025vapo, zheng2025group, chen2025minimax, zhao2025geometric, zhang2025survey}.
\subsection{RLHF-based Post-Training for Text-to-Image Models.} In the early attempts to apply reinforcement learning to the post-training of text-to-image models, RLHF-based techniques constituted the dominant paradigm. Among these, reward-weighted likelihood maximization \citep{lee2023aligning} introduced a three-stage procedure that harnesses RLHF to enhance alignment between generated images and textual prompts. DRAFT \citep{clark2023directly} proposes a streamlined yet effective approach that fine-tunes generative models by maximizing differentiable reward signals. DOODL \citep{wallace2023end} focuses on the optimization of initial diffusion noise vectors by computing losses over fully generated outputs, allowing iterative refinement of individual generations during inference. ReFL \citep{xu2023imagereward} implements a two-stage framework: a reward model, ImageReward, is first trained on human preference data, and subsequently, fine-tuning proceeds with randomly selected timesteps to predict final images, stabilizing the training process and mitigating overfitting to the terminal denoising step. DDPO \citep{blacktraining} formalizes the diffusion denoising process as a Markov Decision Process (MDP) and updates pre-trained models via policy gradient optimization to maximize rewards derived from human feedback. In a similar vein, DPOK \citep{fan2023dpok} combines reinforcement learning with reverse KL regularization to optimize scored rewards, applying this strategy to both RL-based and supervised fine-tuning scenarios. 
\subsection{DPO-Style Post-Training for Text-to-Image Models.} Inspired by the success of DPO in achieving effective alignment of LLMs through implicit reward estimation \citep{rafailov2023direct, lv2025hidden}, several studies have explored similar principles in other domains, grounding their approaches in the concept of implicit reward modeling. Diffusion-DPO \citep{wallace2024diffusion} extends the DPO framework by leveraging the evidence lower bound (ELBO) to formulate a fully differentiable objective, and further approximates the reverse diffusion process using the forward trajectory. D3PO \citep{yang2024using} models the denoising procedure as a multi-step Markov decision process (MDP), demonstrating that directly updating the policy based on human preferences within this MDP is equivalent to first estimating the optimal reward function. Step-by-step preference optimization (SPO) \citep{liang2025aesthetic} introduces step-level evaluation and adjustment, ensuring that preference signals are accurately propagated at each denoising stage. Collectively, the paradigm encompasses three key directions for enhancement. The first avenue focuses on leveraging insights from established generative model alignment techniques, including those developed for LLMs, and systematically modifying them to suit text-to-image model applications \citep{li2024aligning, gu2024diffusion, hong2024margin, sun2025positive, sun2025identical}. The second avenue involves combining the DPO framework with the intrinsic mechanisms of text-to-image models, with particular emphasis on diffusion and flow-matching \citep{tang2024tuning, gambashidze2024aligning, eyring2024reno}. The third avenue focuses on investigating novel training paradigms aimed at enhancing performance and improving efficiency \citep{yuan2024self, park2024direct, karthik2025scalable, huang2025patchdpo, sun2025generalizing, sun2025diffusion, bai2025entropy}.
\subsection{GRPO-Style Post-Training for Text-to-Image Models.} Benefiting from the remarkable success of DeepSeek-R1 \citep{guo2025deepseek}, the application of the GRPO paradigm in text-to-image generation has begun to garner growing interest from the research community. Depending on the type of model being fine-tuned, GRPO-style post-training has developed along three primary technical research pathways.

The first line of research investigates online reinforcement learning optimization in the context of \textbf{diffusion models and flow matching models}; and this also forms the core emphasis of our study. Flow-GRPO \citep{liu2025flow} attains computationally efficient RL-based optimization by introducing stochasticity through ODE-to-SDE conversion and by reducing the number of denoising steps during training. DanceGRPO \citep{xue2025dancegrpo} establishes a unified optimization framework that enhances diffusion models and rectified flows over diverse generative tasks. DiffusionNFT \citep{zheng2025diffusionnft} directly optimizes diffusion models in the forward process through flow matching, where an implicit policy improvement direction is constructed by contrasting positive and negative generations, thereby seamlessly integrating reinforcement signals into a supervised learning objective. TempFlow-GRPO \citep{he2025tempflow} explicitly captures and leverages the temporal structure inherent in flow-based generative processes. MixGRPO \citep{li2025mixgrpo} employs a sliding-window strategy that restricts SDE-based sampling and RL-based optimization to a designated temporal window, while relying on ODE-based sampling outside this interval. ChunkGRPO \citep{luo2025sample} aggregates consecutive timesteps into chunks and performs optimization at the chunk level. Stepwise-Flow-GRPO \citep{savani2026stepwise} assigns step-wise reward gains using intermediate reward estimation and improved SDE sampling to address uniform credit assignment, while DenseGRPO \citep{deng2026densegrpo} introduces step-wise dense rewards and reward-aware timestep-dependent exploration for better alignment. Starting from shared initial noise samples, TreeGRPO \citep{ding2025treegrpo} branches strategically to produce multiple candidate trajectories while leveraging their shared prefixes to reduce redundant computation. BranchGRPO \citep{li2025branchgrpo} organizes rollouts as a branching tree, leveraging shared prefixes to reduce computational cost while pruning low-value paths and redundant branches. Dynamic-TreeRPO \citep{fu2025dynamic} accelerates training convergence and promotes efficient exploration of the search space by introducing dynamically adaptive clipping boundaries guided by reward signals, together with a training strategy that incorporates supervised fine-tuning. Neighbor-GRPO \citep{he2025neighbor} promotes trajectory diversity by introducing controlled perturbations to the initial ODE noise and updates the model via a softmax distance–based surrogate leaping policy. Multi-GRPO \citep{lyu2025multi} enhances multi-objective optimization through two complementary mechanisms: tree-based trajectories that improve credit assignment for early denoising stages via multiple descendant evaluations, and reward-based grouping that decouples advantage estimation across objectives to mitigate reward interference. Pref-GRPO \citep{wang2025pref} employs a pairwise preference–based reward to guide GRPO optimization, transitioning the learning objective from score maximization to preference alignment to achieve enhanced training stability. SuperFlow \citep{chen2025superflow} is a RL training framework for flow-based models that adaptively adjusts group sizes via variance-aware sampling and computes step-level advantages in a manner consistent with continuous-time flow dynamics. GDRO \citep{wang2026gdro} is a sampler-agnostic, offline training method that avoids online sampling and ODE-to-SDE conversions. Flash-DMD \citep{chen2025flash} accelerates convergence by integrating distillation with joint RL-based refinement. DDRL \citep{ye2025data} introduces a novel framework that anchors the policy to an off-policy data distribution through the use of forward KL divergence. UniGRPO \citep{liu2026unigrpo} jointly optimizes reasoning-driven text and image generation by integrating GRPO with an enhanced FlowGRPO, establishing a scalable baseline for post-training interleaved multimodal models. Smart-GRPO \citep{yu2025smart} formulates noise optimization as an iterative search procedure, where candidate perturbations are decoded and evaluated to guide the noise distribution toward higher-reward areas. DiverseGRPO \citep{liu2025diversegrpo} incorporates a distributional creativity bonus together with structure-aware regularization to recalibrate reward signals and align denoising-stage constraints with diversity preservation. BeautyGRPO \citep{yang2026beautygrpo} combines a fine-grained aesthetic reward model with dynamic path-guided RL sampling to align face retouching with human preferences while preserving image fidelity. E-GRPO \citep{zhang2026grpo} restructures diffusion sampling with entropy-aware step merging and group-normalized advantages to address sparse and ambiguous rewards in multi-step denoising. Furthermore, Edit-R1 \citep{li2025uniworld} employs an MLLM as a training-free reward model, using its output logits to deliver fine-grained feedback and adapt the DiffusionNFT framework for image editing tasks; HP-Edit \citep{li2026hp} introduces HP-Scorer, an automatic evaluator aligned with human preferences, which is used as the reward function for post-training the editing model.

The second line of research explores online reinforcement learning optimization for \textbf{autoregressive models}. By employing multi-dimensional reward functions, AR-GRPO \citep{yuan2025ar} refines the outputs of autoregressive models, evaluating generated images in terms of perceptual quality, realism, and semantic fidelity. STAGE \citep{ma2025stage} achieves stable and generalizable training by leveraging advantage/KL reweighting along with an entropy-based reward. GCPO \citep{zhang2025group} facilitates effective policy optimization on critical tokens, which are identified from three perspectives: causal dependency, entropy-induced spatial structure, and token diversity guided by RLVR. Moreover, EARL \citep{ahmadi2025promise} extends similar paradigm to text-guided image editing tasks.

The third line of research investigates online reinforcement learning optimization for \textbf{masked generative models}. Mask-GRPO \citep{luoreinforcement} redefines transition dynamics and models the unmasking process as a multi-step decision-making problem. Co-GRPO \citep{zhou2025co} co-optimizes model and schedule parameters via a shared reward, bypassing expensive backpropagation through multi-step generation.  MaskFocus \citep{zhang2025maskfocus} evaluates the importance of each generation step using image embedding similarity and designates the most informative steps as critical, thereby achieving a balance between computational efficiency and overall performance.

\section{Detailed Results on the UniGenBench++ Benchmark}
\label{Detailed Results on the UniGenBench++ benchmark}
\begin{table}[!h]
\centering
\caption{Detailed Results on the UniGenBench++ Benchmark for \textcolor{myblue}{English Short Prompt Evaluation} and \textcolor{myorange}{English Long Prompt Evaluation}. Gemini 2.5 Pro \citep{comanici2025gemini} is used as the MLLM for evaluation.}
\resizebox{\linewidth}{!}{
\begin{tabular}{lc|c c |cc c c c c |cc c c c c |cc c c |cc |cc c |cc|cc}
\toprule
\rowcolor{myblue!75!white}\multicolumn{29}{c}{\textbf{English Short Prompt Evaluation--Detailed}} \\
\midrule
Method & Overall & Style & \makecell{World\\Know.}
& \multicolumn{6}{c|}{Attribute} 
& \multicolumn{6}{c|}{Action} 
& \multicolumn{4}{c|}{Relationship}
& \multicolumn{2}{c|}{Compound}
& \multicolumn{3}{c|}{Grammar}
& \multicolumn{2}{c|}{Layout}
& \makecell{Logic.\\Reason.}
& Text \\
\cmidrule(lr){5-10}
\cmidrule(lr){11-16}
\cmidrule(lr){17-20}
\cmidrule(lr){21-22}
\cmidrule(lr){23-25}
\cmidrule(lr){26-27}
 &  &  & &
Quant. & Express. & Materi. & Size & Shape & Color
& Hand & \makecell{Full\\Body} & Animal & \makecell{Non\\Contact} & Contact & State
& Compos. & Sim. & Inclus. & Compare.
& Imagin. & \makecell{Feat\\Match.}
& \makecell{Pron\\Ref.} & Consist. & Neg.
& 2D & 3D
&  &  \\
\midrule

FLUX.1 Dev & 61.08 & 86.20 & 87.66 & 57.55 & 94.17 & 50.64 & 60.00 & 67.36 & 63.68 & 68.48 & 51.19 & 87.66 & 51.92 & 75.74 & 60.20 & 61.72 & 68.48 & 65.88 & 27.52 & 44.79 & 62.50 & 44.62 & 74.26 & 64.81 & 67.80 & 49.74 & 27.52 & 37.64 \\

DanceGRPO & 60.64 & 76.30 & 85.28 & 65.57 & 96.67 & 46.79 & 53.12 & 65.97 & 65.57 & 66.30 & 54.17 & 85.28 & 58.33 & 69.12 & 57.14 & 61.72 & 61.41 & 70.61 & 37.61 & 48.18 & 70.83 & 38.85 & 69.12 & 58.80 & 72.73 & 67.09 & 37.61 & 28.16 \\

SLAS & 62.60 & 82.30 & 81.96 & 59.91 & 96.67 & 58.33 & 65.62 & 66.67 & 73.58 & 72.28 & 63.10 & 81.96 & 59.62 & 77.94 & 73.47 & 66.41 & 69.57 & 70.61 & 41.06 & 53.91 & 69.44 & 38.85 & 74.26 & 50.93 & 73.11 & 60.46 & 41.06 & 28.16 \\ \midrule

\rowcolor{myorange!75!white}\multicolumn{29}{c}{\textbf{English Long Prompt Evaluation--Detailed}} \\
\midrule
Method & Overall & Style & \makecell{World\\Know.}
& \multicolumn{6}{c|}{Attribute} 
& \multicolumn{6}{c|}{Action} 
& \multicolumn{4}{c|}{Relationship}
& \multicolumn{2}{c|}{Compound}
& \multicolumn{3}{c|}{Grammar}
& \multicolumn{2}{c|}{Layout}
& \makecell{Logic.\\Reason.}
& Text \\
\cmidrule(lr){5-10}
\cmidrule(lr){11-16}
\cmidrule(lr){17-20}
\cmidrule(lr){21-22}
\cmidrule(lr){23-25}
\cmidrule(lr){26-27}
 &  &  & &
Quant. & Express. & Materi. & Size & Shape & Color
& Hand & \makecell{Full\\Body} & Animal & \makecell{Non\\Contact} & Contact & State
& Compos. & Sim. & Inclus. & Compare.
& Imagin. & \makecell{Feat\\Match.}
& \makecell{Pron\\Ref.} & Consist. & Neg.
& 2D & 3D
&  &  \\
\midrule

FLUX.1 Dev & 69.22 & 90.37 & 89.45 & 79.79 & 63.69 & 80.89 & 85.23 & 73.25 & 86.61 & 66.67 & 65.82 & 67.39 & 56.70 & 56.32 & 69.15 & 65.43 & 64.74 & 79.89 & 68.20 & 67.90 & 65.19 & 81.75 & 69.05 & 60.71 & 80.79 & 75.36 & 47.06 & 35.60 \\
DanceGRPO & 65.72 & 79.24 & 87.28 & 73.94 & 62.85 & 76.96 & 81.25 & 69.93 & 84.14 & 61.86 & 63.29 & 68.12 & 50.89 & 60.34 & 71.23 & 62.88 & 61.86 & 76.72 & 72.33 & 70.66 & 60.75 & 81.35 & 59.52 & 50.71 & 78.81 & 74.46 & 49.02 & 23.91 \\
SLAS & 68.32 & 83.55 & 87.43 & 73.40 & 68.16 & 75.76 & 80.49 & 75.00 & 87.24 & 62.18 & 62.03 & 73.19 & 64.29 & 60.06 & 72.22 & 61.61 & 71.47 & 85.06 & 69.66 & 71.93 & 63.79 & 84.92 & 57.54 & 50.36 & 80.51 & 78.08 & 53.19 & 30.43 \\
\bottomrule
\end{tabular}
}
\label{Detailed UniGenBench++ Benchmark Table}
\end{table}

Building upon the primary results, \Cref{Detailed UniGenBench++ Benchmark Table} provides a comprehensive, fine-grained breakdown across 27 detailed sub-categories. In the short prompt setting, SLAS achieves consistent performance gains across fine-grained tasks, with notable enhancements in action understanding sub-tasks like full-body motion modeling and contact interaction, as well as relationship modeling tasks such as compositional relationship reasoning. These fine-grained improvements serve as the core foundation that supports SLAS's overall performance lead over DanceGRPO in this setting. In the more challenging long prompt setting, SLAS maintains its leading position over DanceGRPO across the majority of sub-dimensions. It particularly excels at long-text-dependent complex tasks, such as logical reasoning and 3D layout modeling, which are critical for real-world applications. These results systematically demonstrate that SLAS effectively improves fine-grained generation under varying prompt lengths, yielding more balanced and reliable performance on complex text-to-image tasks. This further shows that our method enables the model to better capture fine-grained requirements in user prompts.

\section{Additional Qualitative Examples}
\label{Additional Qualitative Examples}
In this section, we present additional qualitative comparisons between SLAS and DanceGRPO. SLAS yields better semantic coherence, spatial consistency, and compositional reasoning, while reducing high-frequency artifacts (e.g., over-stylization and exaggerated contrast) observed in DanceGRPO.

\begin{figure*}[ht]
\centering
\includegraphics[width=1.0\linewidth]{NeurIPS_2026_appendix_1.jpg}
\label{visual_appendix_1}
\end{figure*}
\begin{figure*}[ht]
\centering
\includegraphics[width=1.0\linewidth]{NeurIPS_2026_appendix_2.jpg}
\label{visual_appendix_2}
\end{figure*}
\clearpage
\section{Product-level E-commerce Image Editing Application}
\label{Product-level E-commerce Image Editing Application}

We have implemented the proposed method in a large-scale, real-world commercial setting. Specifically, it is integrated into a product-level digital human livestreaming platform for e-commerce. To support dynamic and high-fidelity product presentation, we develop a specialized model for commodity manipulation and seamless item transitions, enabling coherent and visually consistent product showcases under diverse conditions. The task presents two key challenges. First, it requires high visual fidelity, imposing strict constraints on spatial alignment, motion consistency, and appearance preservation. This is particularly challenging for heterogeneous product categories with large variations in geometry, texture, and scale, making it difficult to maintain both realism and cross-commodity consistency. Second, the problem involves complex human-commodity interaction modeling. The model jointly reason about human pose, hand–commodity contact, and occlusion relationships, while ensuring that manipulated commodity remain physically plausible and semantically consistent. 

These challenges become even more pronounced during reinforcement post-training. Furthermore, we apply Edit-R1 \citep{li2025uniworld} and the proposed method separately during this process for comparison. The reward is designed based on three dimensions: commodity presence, commodity similarity, and hand–commodity interaction realism. The given text prompt is ``The person in Figure 1 is holding the commodity in Figure 2''. Several representative examples are presented as follows.

\begin{figure}[h]
\centering
\includegraphics[width=1.0\linewidth]{NeurIPS_2026_appendix_editing_1.jpg}
\label{visual_appendix_editing_1}
\end{figure}
\clearpage
\begin{figure}[t]
\centering
\includegraphics[width=1.0\linewidth]{NeurIPS_2026_appendix_editing_2.jpg}
\label{visual_appendix_editing_2}
\end{figure}

As shown in these examples, our method  outperforms the baseline Edit-R1 across diverse commodity categories, delivering superior results in complex e-commerce scenarios. The baseline often suffers from critical rendering failures: it frequently produces duplicate commodities, extraneous disembodied hands, stiff or misaligned interaction poses, and distorted commodity appearances, failing to naturally place the target commodity in the person’s hand. In contrast, our method effectively adapts the target commodity to the scene. It properly scales the commodity to fit the person’s hand while preserving the original appearance and fine-grained details, even for commodities with large variations in shape and size. Meanwhile, it generates natural and physically plausible hand–commodity interactions, with accurate contact alignment and semantically consistent poses, resulting in coherent and realistic commodity presentations. Across the three core evaluation dimensions (commodity presence, commodity similarity, and hand–commodity interaction realism), our method achieves significantly better performance, providing a robust solution to the fundamental challenges of high-fidelity commodity presentation in e-commerce digital human livestreaming.

\section{Broader Impacts}
\label{Broader Impacts}
While super-linear advantage shaping demonstrates strong performance, all advances in text-to-image generation carry ethical considerations. In particular, the enhanced models might be misused to generate harmful, hateful, misleading, or sexually explicit content. Therefore, the models will remain private until thorough safety evaluations and internal ethical reviews are completed. We will also urge potential future users to refrain from generating unsafe content.

\end{document}